\documentclass[journal]{IEEEtran}

\usepackage{stmaryrd}
\usepackage{bm}
\usepackage{amssymb}
\usepackage{booktabs}
\usepackage{graphicx}
\usepackage{cite}
\usepackage{float}
\usepackage{graphicx}
\usepackage{stfloats}
\usepackage{amsmath}
\usepackage{epstopdf}
\usepackage{algorithm,algorithmic}
\usepackage{multirow}
\usepackage{framed}
\usepackage[colorlinks,linkcolor=red,anchorcolor=green,citecolor=green]{hyperref}
\usepackage[numbers,sort&compress]{natbib}
\usepackage{amsthm}

\newcommand{\ie}{{\em i.e.}}

\newtheorem{lemma}{Lemma}
\newtheorem{corollary}{Corollary}
\newtheorem{theorem}{Theorem}
\newtheorem{proposition}{Proposition}

\ifCLASSINFOpdf

\else

\fi

\begin{document}
%
\title{Analyzing the Weighted Nuclear Norm Minimization and Nuclear Norm Minimization based on Group Sparse Representation}

\author{Zhiyuan~Zha,~\IEEEmembership{Student Member,~IEEE}, Xin~Yuan,~\IEEEmembership{Senior Member,~IEEE}, Bei~Li, Xinggan~Zhang, Xin~Liu,~\IEEEmembership{Member,~IEEE}, Lan~Tang,~\IEEEmembership{Member,~IEEE}, Ying-Chang~Liang,~\IEEEmembership{Fellow,~IEEE}
\thanks{Z. Zha, B. Li, X. Zhang and L. Tan are with the Department of Electronic Science and Engineering,
Nanjing University, Nanjing 210023, China. E-mail: zhazhiyuan.mmd@gmail.com, njulibei@163.com, zhxg@nju.edu.cn, tanglan@nju.edu.cn.}
\thanks{X. Yuan is with Nokia Bell Labs, 600 Mountain Avenue, Murray Hill, NJ, 07974, USA. E-mail: xyuan@bell-labs.com.}
\thanks{L. Xin is with the Center for Machine Vision and Signal Analysis, University of Oulu, 90014, Finland. E-mail: linuxsino@gmail.com.}
\thanks{Y. C. Liang is with the University of Electronic Science and Technology of China, Chengdu 611731, China. E-mail: Liangyc@ieee.org.}
}

\maketitle

\begin{abstract}
Rank minimization methods have attracted considerable interest in various areas, such as computer vision and machine learning. The most representative work is nuclear norm minimization (NNM), which can recover the matrix rank exactly under some restricted and theoretical guarantee conditions. However, for many real applications, NNM is not able to approximate the matrix rank accurately, since it often tends to over-shrink the rank components. To rectify the weakness of NNM, recent advances have shown that weighted nuclear norm minimization (WNNM) can achieve a better matrix rank approximation than NNM, which heuristically set the weight being inverse to the singular values. However, it still lacks a sound mathematical explanation on why WNNM is more feasible than NNM. In this paper, we propose a scheme to analyze WNNM and NNM from the perspective of the group sparse representation. Specifically, we design an adaptive dictionary to bridge the gap between the group sparse representation and the rank minimization models. Based on this scheme, we provide a mathematical derivation to explain why WNNM is more feasible than NNM. Moreover, due to the heuristical set of the weight, WNNM sometimes pops out error in the operation of SVD, and thus we present an adaptive weight setting scheme to avoid this error. We then employ the proposed scheme on two low-level vision tasks including image denoising and image inpainting. Experimental results demonstrate that WNNM is more feasible than NNM and the proposed scheme outperforms many current state-of-the-art methods.
\end{abstract}

\begin{IEEEkeywords}
Rank minimization, nuclear norm minimization, weighted nuclear norm minimization, group sparse representation, image restoration.
\end{IEEEkeywords}

\IEEEpeerreviewmaketitle

\section{Introduction}

Due to the fact that the data from many practical cases have low rank property, low rank matrix approximation (LRMA), which aims to recover the underlying low rank structure from its degraded/corrupted samples, has a wide range of applications in the area of computer vision and machine learning \cite{1,2,3,4,5,6,7,8,9,10,11,12,13,14,15,16,17,18,19,20,21,22,23,24,25,26,27,33}. For instance, the Netflix customer data matrix is regarded as low rank because the customers' choices are mostly affected by a few common factors \cite{26}. The video clip is captured by a static camera satisfies the "low rank + sparse" structure so that the background modeling can be conducted by the LRMA \cite{16,22}. As the matrix formed by nonlocal similar patches in a natural image is of low rank, a flurry of image completion problems based on low rank models have been proposed, such as image alignment \cite{25}, video denoising \cite{23}, shadow removal \cite{27} and reconstruction of occluded/corrupted face images \cite{5,17,24}.

Generally speaking, methods of LRMA can be classified into two categories: the low rank matrix factorization (LRMF) \cite{1,2,3,5,6,7} and the rank minimization methods \cite{9,10,11,12,15,16}. Given an input matrix $\textbf{\emph{Y}}$, the goal of LRMF is to factorize it into the product of two low rank matrices that can be used to reconstruct the low rank matrix $\textbf{\emph{X}}$ with certain fidelity. Various LRMF-based methods have been proposed, such as the classical SVD under $\ell_2$-norm \cite{3}, robust LRMF methods under $\ell_1$-norm \cite{2,5} and other probabilistic methods \cite{6,7}.

Another parallel research is the rank minimization methods, with the nuclear norm minimization (NNM) \cite{9,10} being the representative one. The nuclear norm of a matrix $\textbf{\emph{X}}$, denoted by $||\textbf{\emph{X}}||_*$, is the sum of its singular values, \ie, $||\textbf{\emph{X}}||_*=\sum\nolimits_i{\boldsymbol\sigma_i}$, where ${\boldsymbol\sigma_i}$ is the $i$-th singular value of the matrix $\textbf{\emph{X}}$.  The goal of NNM is to recover the underlying low rank matrix $\textbf{\emph{X}}$ from its degraded observation matrix $\textbf{\emph{Y}}$, by minimizing $||\textbf{\emph{X}}||_*$. In recent years,  a series of applications based on NNM have been proposed, such as video denoising \cite{23}, background extraction \cite{5,25,27} and subspace clustering \cite{1}. However, the nuclear norm is usually adopted as a convex surrogate of the matrix rank. Although enjoying the theoretical guarantee, the singular value thresholding (SVT) model \cite{10} for NNM tends to over-shrink the rank components, as it treats the different rank components equally, and thus it cannot estimate the matrix rank accurately enough. To improve the performance of NNM, numerous methods have been proposed \cite{11,13,14,15,21,22,28}. For instance, inspired by the success of $\ell_p$ (0$<p<$1) sparse optimization \cite{100,101,102,103}, Schatten $p$-norm is proposed \cite{21,28}, which is defined as the $\ell_p$-norm (0$<p<$1) of the singular values. Compared with the nuclear norm, Schatten $p$-norm not only achieves a more accurate recovery result, but also requires only a weaker restricted isometry \cite{28}. The truncated nuclear norm regularization (TNNR) \cite{13} and the partial sum minimization (PSM) \cite{14} keep the largest $r$ singular values unchanged and only minimize the smallest $N-r$ ones, where $N$ is the number of the singular values and $r$ is the rank of the matrix. Inspired by the singular values have clear physical meanings, Gu $\emph{et al}.$ \cite{15} proposed the weighted nuclear norm minimization (WNNM) model. Recently, Xie $\emph{et al}.$ \cite{22} proposed an improved WNNM model, namely, weighted schatten $p$-norm minimization (WSNM) for low rank matrix approximation.

According to the above analysis, to the best of our knowledge, the most well-known one is the WNNM model. However, it is still lack of a sound mathematical explanation why WNNM is more feasible than NNM.  Bearing the above concern in mind, in this paper, we propose a scheme to analyze  WNNM and NNM from the point of the group sparse representation (GSR). To be concrete, an adaptive dictionary learning method is designed to bridge the gap between the GSR and the rank minimization models. Based on this adaptive dictionary, we prove that NNM and WNNM are equivalent to the  $\ell_1$-norm minimization based on GSR and the weighted $\ell_1$-norm minimization based on GSR, respectively. Following this, based on this scheme, a mathematical derivation is introduced to explain why WNNM is more feasible than NNM. In addition, because of the heuristical set of the weight in WNNM, it sometimes pops out error in the operation of SVD, and therefore an adaptive weight setting scheme is presented to avoid this error. We apply the proposed scheme to solve two low-level vision tasks, \ie, image denoising and image inpainting. Experimental results demonstrate that WNNM is more feasible than NNM and the proposed scheme outperforms many current state-of-the-art methods in both the objective and the perceptual qualities.

The rest of this paper is organized as follows. Section~\ref{sec:2} introduces  the related works including the weighted $\ell_1$-norm minimization for sparse representation, nuclear norm minimization, weighted nuclear norm minimization and group sparse representation. Section~\ref{sec:3} presents a scheme to analyze WNNM and NNM from the perspective of group sparse representation and proves that why WNNM is more feasible than NNM. Section~\ref{sec:4} introduces WNNM model for two low-level vision tasks, \ie, image denoising and image inpainting. Section~\ref{sec:5} reports the experimental results. Finally, section~\ref{sec:6} concludes this paper.

\section {Related Works}
\label{sec:2}
In this paper, we provide a mathematical explanation why WNNM is more feasible than NNM. To this end, we will introduce some related works of the weighted $\ell_1$-norm minimization for sparse representation, NNM, WNNM and group sparse representation models in this section. We firstly introduce the weighted $\ell_1$-norm minimization for sparse representation.
\subsection {Weighted $\ell_1$-Norm Minimization for Sparse Representation}
\label{sec:2.1}

Sparse representation model has been successfully used in various applications, such as compressive sensing \cite{29}, face recognition \cite{30} and image restoration \cite{31}. Mathematically, it can be represented by solving the following minimization problem,
\begin{equation}
\min_{\textbf{\emph{x}}\in{\mathbb R}^n} ||\textbf{\emph{x}}||_0, \ \ \ \ \ \ \ s.t., \ \ \ \ \ \textbf{\emph{y}}= \boldsymbol\phi  \textbf{\emph{x}}
\label{eq:1}
\end{equation} 
where $\textbf{\emph{y}}$ is an $m \times 1$ vector and $\boldsymbol\phi$ is an $m\times n$ redundant matrix with $m\leq n$. $||\cdot||$ represents the $\ell_0$-norm, counting the number of non-zero entries of $\textbf{\emph{x}}$.

However, since $\ell_0$-norm minimization problem is a difficult combinatorial optimization,  solving Eq.~\eqref{eq:1} is NP-hard. Therefore,  $\ell_0$-norm minimization is often relaxed to the convex $\ell_1$-norm minimization problem. Specifically, by selecting an appropriate regularization parameter $\lambda$, Eq.~\eqref{eq:1} can be rewritten as the following unconstrained optimization problem,
\begin{equation}
\hat{\textbf{\emph{x}}} = \arg\min_{\textbf{\emph{x}}}\frac{1}{2}||\textbf{\emph{y}}-\boldsymbol\phi \textbf{\emph{x}}||_2^2 +\lambda||\textbf{\emph{x}}||_1
\label{eq:2}
\end{equation} 

However, in some practical problems, such as  image inverse problems \cite{102,32,41}, $\ell_1$-norm minimization is quite hard to achieve a sparsity solution accurately. This raises the question of whether we can improve the sparsity of $\ell_1$-norm minimization. In other words, we wish that $\ell_1$-norm can alternative to $\ell_0$-norm and discover a better solution. For this reason, we introduce a well-known norm minimization method, \ie, weighted $\ell_1$-norm minimization \cite{32}, and instead of Eq.~\eqref{eq:2}, we have the following minimization problem,
\begin{equation}
\hat{\textbf{\emph{x}}} = \arg\min_{\textbf{\emph{x}}}\frac{1}{2}||\textbf{\emph{y}}-\boldsymbol\phi \textbf{\emph{x}}||_2^2 + \lambda||\textbf{\emph{w}}\textbf{\emph{x}}||_1
\label{eq:3}
\end{equation} 
where $\textbf{\emph{w}}$ is a weight assigned to $\textbf{\emph{x}}$ and it can enhance the representation capability of $\textbf{\emph{x}}$. In general, each value of the weight $\textbf{\emph{w}}$ is inverse proportion to each value of $\textbf{\emph{x}}$ \cite{32}. Also, we have the following conclusion.

\begin{proposition}
\label{proposition:1}
The weighted $\ell_1$-norm minimization can enhance the sparsity performance in comparison with traditional $\ell_1$-norm minimization, \ie,
\begin{equation}
\min_{\textbf{\emph{x}}\in{\mathbb R}^n} ||{\textbf{\emph{w}}}\textbf{\emph{x}}||_1 \succ \min_{\textbf{\emph{x}}\in{\mathbb R}^n} ||\textbf{\emph{x}}||_1
\label{eq:4}
\end{equation} 
where ${ v_1}\succ { v_2}$ denotes that the entry ${ v_1}$ has  much more sparsity encouraging than the entry ${ v_2}$.

\end{proposition}
\begin{proof}
 We consider the log penalty function ${\rm log}(\textbf{\emph{x}}+\epsilon)$ as the regularization term and we have,
\begin{equation}
\hat{\textbf{\emph{x}}} = \arg\min_{\textbf{\emph{x}}}\frac{1}{2}||\textbf{\emph{y}}-\boldsymbol\phi \textbf{\emph{x}}||_2^2 +   \lambda{\rm log}(\textbf{\emph{x}}+\epsilon),
\label{eq:5}
\end{equation} 
where $\epsilon$ denotes a small constant. Note that this function ${\rm log}(\textbf{\emph{x}}+\epsilon)$ approximates the sum of the logarithm of $\textbf{\emph{x}}$, and thus it is smooth yet non-convex. Fig.~\ref{fig:1} shows the non-convex log penalty function ${\rm log}(\textbf{\emph{x}}+\epsilon)$ and the $\ell_1$-norm in the scalar case. One can clearly observe that the log penalty function ${\rm log}(\textbf{\emph{x}}+\epsilon)$ is more accurate to approximate canonical $\ell_0$-norm (Eq.~\eqref{eq:1} ) than $\ell_1$-norm.

\begin{figure}[!htbp]
\vspace{-3mm}
		\centering
		{\includegraphics[width=.5\textwidth]{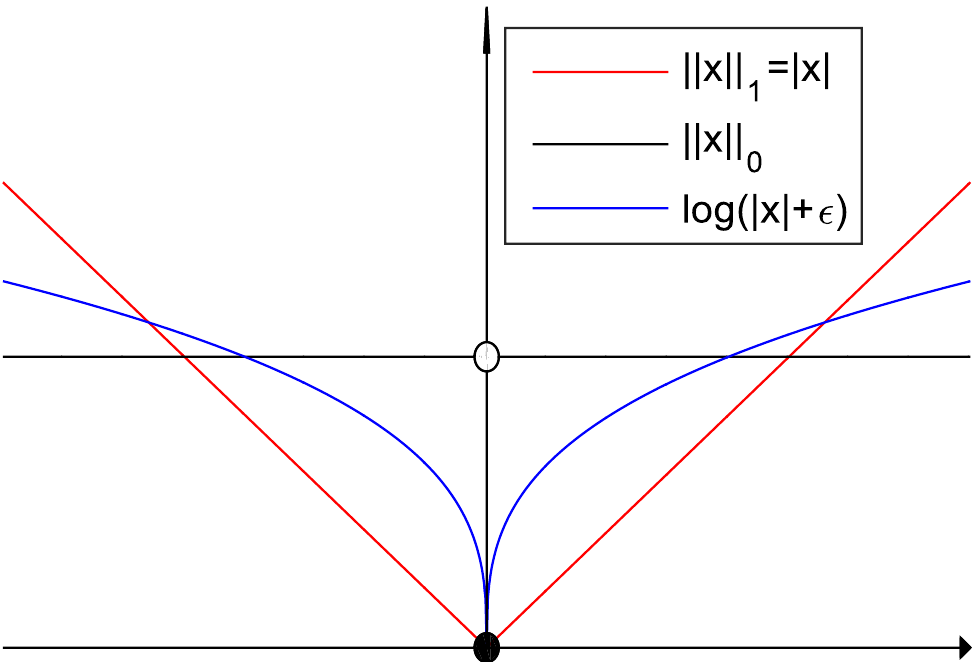}}
\vspace{-3mm}
	\caption{Comparison of log penalty function ${\rm log}(\textbf{\emph{x}}+\epsilon)$, $\ell_1$-norm: $||\textbf{\emph{x}}||_1$ and $\ell_0$-norm: $||\textbf{\emph{x}}||_0$ in the scalar case.}
	\label{fig:1}
	\vspace{-4mm}
\end{figure}
Although ${\rm log}(\textbf{\emph{x}}+\epsilon)$ is non-convex, we can solve it efficiently by a local minimization method. Specifically, Let $\textbf{\emph{R}}(\textbf{\emph{x}}) = {\rm log}(\textbf{\emph{x}}+\epsilon)$, which can be approximated by using the first-order Taylor expansion, \ie,
\begin{equation}
\textbf{\emph{R}}(\textbf{\emph{x}}) = \textbf{\emph{R}}(\textbf{\emph{x}}^{(t)}) + \langle \nabla\textbf{\emph{R}}(\textbf{\emph{x}}^{(t)}), \textbf{\emph{x}}- \textbf{\emph{x}}^{(t)}\rangle
\label{eq:6}
\end{equation} 
where $\textbf{\emph{x}}^{(t)}$ is the solution obtained in the $t$-th iteration and $\nabla\textbf{\emph{R}}(\textbf{\emph{x}}^{(t)}) = ({1}/({|\textbf{\emph{x}}^{(t)}|+\epsilon})$.

Then, we ignore the constants in Eq.~\eqref{eq:6}, and Eq.~\eqref{eq:5} can be solved by the following equation,
\begin{equation}
\hat{\textbf{\emph{x}}}^{(t+1)} = \arg\min_{\textbf{\emph{x}}}\frac{1}{2}||\textbf{\emph{y}}-\boldsymbol\phi \textbf{\emph{x}}||_2^2 +   \lambda\frac{|\textbf{\emph{x}}|}{|\textbf{\emph{x}}^{(t)}|+\epsilon},
\label{eq:7}
\end{equation} 

Let $\textbf{\emph{w}}^{(t)} = ({1}/{|\textbf{\emph{x}}^{(t)}|+\epsilon}$), Eq.~\eqref{eq:7} can be rewritten as
\begin{equation}
\hat{\textbf{\emph{x}}}^{(t+1)} = \arg\min_{\textbf{\emph{x}}}\frac{1}{2}||\textbf{\emph{y}}-\boldsymbol\phi \textbf{\emph{x}}||_2^2 + \lambda    |\textbf{\emph{w}}^{(t)}\textbf{\emph{x}}|_1,
\label{eq:8}
\end{equation} 

Obviously, the weight $\textbf{\emph{w}}$ is inverse proportion to $\textbf{\emph{x}}$, and we omit the subscript without confusion, we have proven that the weighted $\ell_1$-norm minimization can enhance the sparsity performance in comparison with traditional $\ell_1$-norm minimization. More details about the weighted $\ell_1$-norm minimization for sparse representation, please see \cite{32}.
\end{proof}

\subsection {Nuclear Norm Minimization}
\label{sec:2.2}

According to \cite{9}, the nuclear norm is the tightest convex relaxation of the original rank minimization problem. Given a data matrix $\textbf{\emph{Y}}\in\Re^{m\times k}$, NNM aims to find a matrix $\textbf{\emph{X}}\in\Re^{m\times k}$ of rank $r$,  which satisfies the following objective function,
\begin{equation}
\mathcal{D}_\lambda(\textbf{\emph{Y}}) =\arg\min_{\textbf{\emph{X}}}\frac{1}{2}||\textbf{\emph{Y}}-\textbf{\emph{X}}||_F^2 +\lambda||\textbf{\emph{X}}||_*
\label{eq:9}
\end{equation} 
where $||\textbf{\emph{X}}||_*=\sum\nolimits_i{\boldsymbol\sigma_i}$, and ${\boldsymbol\sigma_i}$ is the $i$-th singular value of the matrix $\textbf{\emph{X}}$.  $\lambda$ is a positive constant. Cand{\`e}s $\emph{et al}.$ \cite{33} demonstrated that the low rank matrix can be perfectly recovered from the degraded/corrupted data matrix with high probability by solving an NNM problem. Cai $\emph{et al}.$ \cite{10} proved that NNM problem can be solved by a soft-thresholding operator efficiently, namely, the solution of Eq.~\eqref{eq:9} which can be solved by
\begin{equation}
\mathcal{D}_\lambda(\textbf{\emph{Y}})=\textbf{\emph{U}}\mathcal{D}_{\lambda}({\boldsymbol\Sigma}){\textbf{\emph{V}}}^T
\label{eq:10}
\end{equation}
where $\textbf{\emph{Y}}=\textbf{\emph{U}}{\boldsymbol\Sigma}{\textbf{\emph{V}}}^T$ is the SVD of $\textbf{\emph{Y}}$ and $\mathcal{D}_\lambda({\boldsymbol\Sigma})$ is the soft-thresholding operator function on diagonal matrix ${\boldsymbol\Sigma}$ with parameter $\lambda$. For each diagonal element ${\boldsymbol\Sigma}_{ii}$ in ${\boldsymbol\Sigma}$, there is $\mathcal{D}_\lambda{({\boldsymbol\Sigma})}_{ii}={\rm soft}({\boldsymbol\Sigma}_{{ii}},\lambda)={\rm
max}({\boldsymbol\Sigma}_{{ii}}-\lambda,0)$. Also, they proved the following theorem.

\begin{theorem}
\label{theorem:1}
For each $\lambda\geq 0$ and $\textbf{\emph{Y}}$, the singular value shrinkage operator Eq.~\eqref{eq:10} \emph{obeys} Eq.~\eqref{eq:9}.
\end{theorem}
\begin{proof}
See \cite{10}.
\end{proof}

\subsection{Weighted Nuclear Norm Minimization}
\label{sec:2.3}

Though a good theoretical guarantee by the singular value thresholding (SVT) model \cite{10}, NNM tends to over-shrink the rank components, and thus it cannot estimate the matrix rank accurately enough. To enforce the low rank regularization efficiently, Gu $\emph{et al}.$ \cite{15} proposed the weighted nuclear norm minimization (WNNM) model. Specifically, the weighted nuclear norm $||\textbf{\emph{X}}||_{{\textbf{\emph{w}}},*}$  is used to regularize $\textbf{\emph{X}}$, and Eq.~\eqref{eq:9} can be rewritten as
\begin{equation}
\mathcal{D}_{{\textbf{\emph{w}}}} (\textbf{\emph{Y}})=\arg\min_\textbf{\emph{X}}\frac{1}{2}||\textbf{\emph{Y}}-\textbf{\emph{X}}||_F^2 + ||\textbf{\emph{X}}||_{{\textbf{\emph{w}}},*}
\label{eq:11}
\end{equation} 
where $||\textbf{\emph{X}}||_{{\textbf{\emph{w}}},*}=\sum\nolimits_i{{{{w}}}_i}{\boldsymbol\sigma_i}$,  ${\textbf{\emph{w}}}=[{{{w}}}_1, {{{w}}}_2,...,{{{w}}}_i]$ and ${{{w}}}_i>0$ is a non-negative weight assigned to ${\boldsymbol\sigma_i}$. Moreover, the following theorem is given.

\begin{theorem}
\label{theorem:2}
If the singular values $\boldsymbol\sigma_1\geq ...\geq\boldsymbol\sigma_{n_0}$ and the weights satisfy $0\leq{{ {w}}}_1\leq...\leq{{{w}}}_{n_0}$, $n_0={\rm min}{(m, k)}$,  WNNM problem in Eq.~\eqref{eq:11} has a globally optimal solution,
\begin{equation}
\mathcal{D}_{\emph{\textbf{{w}}}} (\textbf{\emph{Y}})=\textbf{\emph{U}}\mathcal{D}_{\emph{\textbf{{w}}}} ({\boldsymbol\Sigma}){\textbf{\emph{V}}}^T
\label{eq:12}
\end{equation} 
where $\textbf{\emph{Y}}=\textbf{\emph{U}}{\boldsymbol\Sigma}{\textbf{\emph{V}}}^T$ is the SVD of $\textbf{\emph{Y}}$ and $\mathcal{D}_{{\textbf{\emph{w}}}} ({\boldsymbol\Sigma})$ is the generalized soft-thresholding operator with the weighted vector ${\textbf{\emph{w}}}$, i.e., $\mathcal{D}_{{\textbf{\emph{w}}}} {({\boldsymbol\Sigma})}_{ii}={\rm soft}({\boldsymbol\Sigma}_{{ii}}, {w}_i)={\rm
max}({\boldsymbol\Sigma}_{{ii}}-{w}_i,0)$.
\end{theorem}
\begin{proof}
See \cite{15}.
\end{proof}

To analyze WNNM is more feasible than NNM below, we have the following theorem.
\begin{theorem}
\label{theorem:3}
     For $0\leq{{{ w}}}_1\leq... \leq{{{ w}}}_{n_0}$ and $\textbf{\emph{Y}}$, $n_0={\rm min}{(m, k)}$, the singular value shrinkage operator Eq.~\eqref{eq:12} satisfies Eq.~\eqref{eq:11}.
\end{theorem}
\begin{proof}
For fixed weight $\textbf{\emph{W}}$, $h_0(\textbf{\emph{X}})=\frac{1}{2} ||\textbf{\emph{Y}}-\textbf{\emph{X}}||_F^2 + ||\textbf{\emph{X}}||_{\emph{\textbf{{w}}},*}$, $\hat{\textbf{\emph{X}}}$ minimizes $h_0$ if and only if it satisfies the following optimal condition,
\begin{equation}
\textbf{0}\in \hat{\textbf{\emph{X}}} - \textbf{\emph{Y}} +  \partial ||\hat{\textbf{\emph{X}}}||_{\emph{\textbf{{w}}},*}
\label{eq:13}
\end{equation} 
where $\partial ||\hat{\textbf{\emph{X}}}||_{\emph{\textbf{{w}}},*}$ is the set of subgradients of the weighted nuclear norm. Let matrix $\textbf{\emph{X}} \in\Re^{m \times k}$ be an arbitrary matrix and its SVD be $\textbf{\emph{U}}{\boldsymbol\Sigma}{\textbf{\emph{V}}}^T$. It is known from \cite{33,34} that the subgradient of $||{\textbf{\emph{X}}}||_{\emph{\textbf{{w}}},*}$ can be derived as
\begin{equation}
\begin{aligned}
&\partial ||\textbf{\emph{X}}||_{\emph{\textbf{{w}}},*}=\{\textbf{\emph{U}}{\textbf{\emph{W}}}_r{\textbf{\emph{V}}}^T+\textbf{\emph{Z}}:
\textbf{\emph{Z}}\in\Re^{m \times k}, {\textbf{\emph{U}}}^T\textbf{\emph{Z}} = \textbf{0},\\
 &\textbf{\emph{Z}}\textbf{\emph{V}}=\textbf{0}, \boldsymbol\sigma_j (\textbf{\emph{Z}})\leq { w}_{r+j}, j=1, ..., n_0-r \}
 \label{eq:14}
\end{aligned}
\end{equation}
where $r$ is the rank of $\textbf{\emph{X}}$ and ${\textbf{\emph{W}}}_r$ is the diagonal matrix composed of the first $r$ rows and $r$ columns of the diagonal matrix diag ({\textbf{\emph{W}}}). Note that the sub-gradient conclusion in Eqs.~\eqref{eq:13} and ~\eqref{eq:14} have been proved for convex problem \cite{10}. Nonetheless, they are feasible for non-convex problem \cite{35}.

In order to show that $\hat{\textbf{\emph{X}}}$ obeys Eq.~\eqref{eq:14}, the SVD of $\textbf{\emph{Y}}$ can be rewritten as
\begin{equation}
\textbf{\emph{Y}}=\textbf{\emph{U}}_0\boldsymbol\Sigma_0\textbf{\emph{V}}_0^T+ \textbf{\emph{U}}_1\boldsymbol\Sigma_1\textbf{\emph{V}}_1^T
\label{eq:15}
\end{equation} 
where $\textbf{\emph{U}}_0, \textbf{\emph{V}}_0 ({\rm resp}. \ \textbf{\emph{U}}_1, \textbf{\emph{V}}_1)$ are the singular vectors associated with singular values greater than ${ w}_j$ (resp. smaller than or equal to ${ w}_j$). With these notations, we have
\begin{equation}
\hat{\textbf{\emph{X}}}= \textbf{\emph{U}}_0 (\boldsymbol\Sigma_0-{\textbf{\emph{W}}}_r)\textbf{\emph{V}}_0^T
\label{eq:16}
\end{equation} 

Therefore,
\begin{equation}
\textbf{\emph{Y}}-\hat{\textbf{\emph{X}}}=  \textbf{\emph{U}}_0{\textbf{\emph{W}}}_r\textbf{\emph{V}}_0^T+ \textbf{\emph{Z}}
\label{eq:17}
\end{equation} 
where $\textbf{\emph{Z}}=\textbf{\emph{U}}_1\boldsymbol\Sigma_1 \textbf{\emph{V}}_1^T$, by definition, $\textbf{\emph{U}}_0^T \textbf{\emph{Z}}=\textbf{0},  \textbf{\emph{Z}}\textbf{\emph{V}}_0=\textbf{0}$ and since the diagonal elements of $\boldsymbol\Sigma_1$ are smaller than ${ w}_{j+r}$, it is easy to verify that $\boldsymbol\sigma_j(\textbf{\emph{Z}}) \leq { w}_{r+j}, j=1, 2, ..., n_0-r $. Thus, $\textbf{\emph{Y}}-\hat{\textbf{\emph{X}}}\in \partial||\hat{\textbf{\emph{X}}}||_{\emph{\textbf{{w}}},*}$, which concludes the proof.
\end{proof}

Note that the relationship between Theorem~\ref{theorem:3} is the necessary and sufficient condition of Theorem~\ref{theorem:2}.

\subsection{Group Sparse Representation}
\label{sec:2.4}
Traditional patch-based sparse representation model assumes that each patch of an image can be precisely encoded as a linear combination of basic elements \cite{36,31}. These elements are called as atoms and they compose a dictionary. Patch-based sparse representation model has been successfully used in various image processing and computer vision tasks \cite{37,38}. However, patch-based sparse representation model usually suffers from some limits, such as dictionary learning with great computational complexity and neglecting the correlations between sparsely-coded patches \cite{40}.

Instead of using single patch as the basic unit in sparse representation, recent studies have shown that group sparse representation (GSR) models have demonstrated great potentials in various image processing tasks \cite{39,40,41,60}. The GSR is a powerful mechanism to integrate local sparsity and nonlocal similarity of image patches. To be concrete, an image $\textbf{\emph{x}}$ with size $N$ is divided into $n$ overlapped patches $\textbf{\emph{x}}_i$ of size $\sqrt{m} \times \sqrt{m}$, $i=1, 2, ..., n$. Then, for each exemplar patch  $\textbf{\emph{x}}_i$, its most similar $k$ patches are selected from an $L \times L$ sized searching window to form a set $\textbf{\emph{S}}_i$. After this, all the patches in $\textbf{\emph{S}}_i$ are stacked into a data matrix $\textbf{\emph{X}}_i \in{\mathbb R}^{m\times k}$, which contains each element of $\textbf{\emph{S}}_i$ as its column, \ie, ${\textbf{\emph{X}}}_i=\{{\textbf{\emph{x}}}_{i,1}, {\textbf{\emph{x}}}_{i,2},...,{\textbf{\emph{x}}}_{i,k}\}$. This matrix $\textbf{\emph{X}}_i$ consisting of patches with similar structures is thereby called a group,  where $\{{\textbf{\emph{x}}}_{i,j}\}_{j=1}^k$ denotes the $j$-th patch in the $i$-th group.  Similar to patch-based sparse representation \cite{36,31}, given a dictionary ${\textbf{\emph{D}}}_{i}$, each group ${\textbf{\emph{X}}}_i$ can be sparsely represented by solving the following $\ell_0$-norm minimization problem,
\begin{equation}
\hat{\textbf{\emph{A}}}_i=\arg\min_{{\textbf{\emph{A}}}_i} \sum\nolimits_{i=1}^n \left(\frac{1}{2}||{\textbf{\emph{X}}}_i-{\textbf{\emph{D}}}_i{\textbf{\emph{A}}}_i||_F^2+\lambda||{\textbf{\emph{A}}}_i||_0\right),
\label{eq:18}
\end{equation} 
where ${\textbf{\emph{A}}}_i$ represents the group sparse coefficient of each group ${\textbf{\emph{X}}}_i$. $||\bullet||_F^2$ denotes the Frobenius norm, and $||\bullet||_0$ signifies the $\ell_0$-norm, \ie, counting the nonzero entries of each column in ${\textbf{\emph{A}}}_i$.

\section{Why WNNM is more feasible than NNM}
\label{sec:3}
Recent advances have shown that WNNM can obtain a better matrix rank approximation than NNM \cite{15,16,18,19,20,22}. However, to the best of our knowledge, it still lacks a sound mathematical explanation on why WNNM is more feasible than NNM. In this section, we propose a scheme to analyze WNNM and NNM from the perspective of group sparse representation (GSR). Specifically, we design an adaptive dictionary to bridge the gap between the GSR and the rank minimization models. Based on this adaptive dictionary, we prove that  NNM and WNNM is equivalent to the $\ell_1$-norm minimization based on GSR and the weighted $\ell_1$-norm minimization based on GSR, respectively. Following this, based on above conclusion and Proposition~\ref{proposition:1}, we provide a mathematical explanation why WNNM is more feasible than NNM. We firstly introduce the adaptive dictionary learning.

\subsection{Adaptive Dictionary Learning}
\label{sec:3.1}
 In this subsection, an adaptive dictionary learning method is now designed, that is, for each group ${\textbf{\emph{X}}}_i$, its adaptive dictionary can be learned from its observation ${\textbf{\emph{Y}}}_i\in\mathbb{R}^{m \times k}$.
Specifically, we apply the SVD to ${\textbf{\emph{Y}}}_i$,
\begin{equation}
{\textbf{\emph{Y}}}_i= {\textbf{\emph{U}}}_i{\boldsymbol\Delta}_i{\textbf{\emph{V}}}_i^T=\sum\nolimits_{j=1}^{n_1} \delta_{i,j}{\textbf{\emph{u}}}_{i,j}{\textbf{\emph{v}}}_{i,j}^T,
\label{eq:19}
\end{equation}
where ${\boldsymbol\Delta}_i={\rm diag}(\delta_{i,1}, \delta_{i,2},..., \delta_{i,{n_1}})$ is a diagonal matrix, ${n_0}={\rm min}(m, k)$,  and  ${\textbf{\emph{u}}}_{i,j}, {\textbf{\emph{v}}}_{i,j}$ are the columns of ${\textbf{\emph{U}}}_i$ and ${\textbf{\emph{V}}}_i$, respectively.

Following this, we define each dictionary atom $\textbf{\emph{d}}_{i,j}$ of the adaptive dictionary $\textbf{\emph{D}}_i$ for each group $\textbf{\emph{Y}}_i$, namely,
\begin{equation}
\textbf{\emph{d}}_{i,j}={\textbf{\emph{u}}}_{i,j}{\textbf{\emph{v}}}_{i,j}^T, \ \ \ j=1,2,...,{n_0}.
\label{eq:20}
\end{equation}

Till now, we have learned an adaptive dictionary, \ie,
\begin{equation}
\textbf{\emph{D}}_i=[\textbf{\emph{d}}_{i,1},\textbf{\emph{d}}_{i,2},...,\textbf{\emph{d}}_{i,{n_0}}]
\label{eq:21}
\end{equation}

It can be seen that the proposed dictionary learning method is efficient since it only requires one SVD operation per group.

\subsection{WNNM is more feasible than NNM}
\label{sec:3.2}
In this subsection, we provide a mathematical explanation why WNNM is more feasible than NNM. To this end, based on above designed dictionary, we prove that NNM and WNNM is equivalent to the $\ell_1$-norm minimization based on GSR and the weighted $\ell_1$-norm minimization based on GSR, respectively. Following this, based on this conclusion and Proposition~\ref{proposition:1}, we prove that WNNM is more feasible than NNM. Specifically, for the NNM model, Eq.~\eqref{eq:9} can be rewritten as
\begin{equation}
\begin{aligned}
& \mathcal{D}_{\lambda} (\textbf{\emph{Y}}) = \arg\min_{\textbf{\emph{X}}} \left(\frac{1}{2}||{\textbf{\emph{Y}}} -{\textbf{\emph{X}}}||_F^2+ \lambda||\textbf{\emph{X}}||_{*}\right),\\
&=\arg\min\limits_{\textbf{\emph{X}}_i}\sum\nolimits_{i=1}^n\left(\frac{1}{2}||{\textbf{\emph{Y}}}_i-{\textbf{\emph{X}}}_i||_F^2 +
 \lambda||\textbf{\emph{X}}_i||_{*}\right),
\label{eq:21.1}
\end{aligned}
\end{equation} 

For the WNNM model, the weighted norm $||\textbf{\emph{X}}||_{\emph{\textbf{{w}}},*}$  is used to regularize $\textbf{\emph{X}}$ and we have,
\begin{equation}
\begin{aligned}
& \mathcal{D}_{{\textbf{\emph{w}}}} (\textbf{\emph{Y}}) = \arg\min_{\textbf{\emph{X}}} \left(\frac{1}{2}||{\textbf{\emph{Y}}} -{\textbf{\emph{X}}}||_F^2+||\textbf{\emph{X}}||_{ \emph{\textbf{w}},*}\right),\\
&=\arg\min\limits_{\textbf{\emph{X}}_i}\sum\nolimits_{i=1}^n\left(\frac{1}{2}||{\textbf{\emph{Y}}}_i-{\textbf{\emph{X}}}_i||_F^2 +
 ||\textbf{\emph{X}}_i||_{ {\emph{\textbf{w}}_i},*}\right),
\label{eq:22}
\end{aligned}
\end{equation} 

In order to prove that NNM and  WNNM are equivalent to the $\ell_1$-norm minimization based on GSR and the weighted $\ell_1$-norm minimization based on GSR, respectively, we firstly give the following lemmas.
\begin{lemma}
\label{lemma:1}
The minimization problem
\begin{equation}
\hat{\textbf{\emph{x}}}=\arg\min_{\textbf{\emph{x}}} \left(\frac{1}{2}||{\textbf{\emph{x}}}-{\textbf{\emph{a}}}||_2^2+\tau||{\textbf{\emph{x}}}||_1\right)
\label{eq:23}
\end{equation} 
has a closed-form solution, \ie,
\begin{equation}
\hat{\textbf{\emph{x}}}={\rm soft}({\textbf{\emph{a}}},\tau)= {\rm sgn}({\textbf{\emph{a}}})\odot{\rm max}({\rm abs}({\textbf{\emph{a}}})-\tau,0),
\label{eq:24}
\end{equation} 
where $\odot$ denotes the element-wise (Hadamard) product.
\end{lemma}
\begin{proof}
See \cite{42}.
\end{proof}

\begin{lemma}
\label{lemma:2}
For the following optimization problem
\begin{equation}
\min_{{{\emph{x}}}_i\geq0} \sum\nolimits_{i=1}^n \left(\frac{1}{2}({{\emph{x}}}_i- {{\emph{a}}}_i)^2 + \emph{w}_i \emph{x}_i\right)
\label{eq:25}
\end{equation} 
{If} $ {a}_1 \geq ... \geq {a}_n \geq 0$ {and the weights satisfy} $0\leq{{\emph{w}}}_1\leq...{{\emph{w}}}_{n}$, {then the global optimum of} Eq.~\eqref{eq:25} is
\begin{equation}
\hat{\emph{x}}_i= {\rm soft}(\emph{a}_i,\emph{w}_i)={\rm {max}} (\emph{a}_i-\emph{w}_i, 0)
\label{eq:25.1}
\end{equation} 
\end{lemma}
\begin{proof}
See \cite{16}.
\end{proof}

Now, recalling the adaptive dictionary defined in Eq.~\eqref{eq:20}, and let us come back to the GSR model in Eq.~\eqref{eq:18}.  The $\ell_1$-norm minimization based on GSR can be represented as
\begin{equation}
\hat{\textbf{\emph{A}}}_i=\arg\min_{{\textbf{\emph{A}}}_i} \sum\nolimits_{i=1}^n\left(\frac{1}{2}||{\textbf{\emph{Y}}}_i-{\textbf{\emph{D}}}_i{\textbf{\emph{A}}}_i||_F^2+ \lambda||{\textbf{\emph{A}}}_i||_1\right),
\label{eq:26}
\end{equation} 
According to the above design of the adaptive dictionary ${\textbf{\emph{D}}}_i$ in Eq.~\eqref{eq:20}, we have the following Lemma.
\begin{lemma}
\label{lemma:3}
\begin{equation}
||{\textbf{\emph{Y}}}_i-{\textbf{\emph{X}}}_i||_F^2=||{\textbf{\emph{B}}}_i-{\textbf{\emph{A}}}_i||_F^2,
\label{eq:27}
\end{equation} 
where ${\textbf{\emph{Y}}}_i={\textbf{\emph{D}}}_i{\textbf{\emph{B}}}_i$ and ${\textbf{\emph{X}}}_i={\textbf{\emph{D}}}_i{\textbf{\emph{A}}}_i$.
\end{lemma}
\begin{proof}

The adaptive dictionary ${\textbf{\emph{D}}}_i$ is constructed by Eq.~\eqref{eq:20}. From the unitary property of ${\textbf{\emph{U}}}_i$ and ${\textbf{\emph{V}}}_i$, we have
\begin{equation}
\begin{aligned}
&||{\textbf{\emph{Y}}}_i-{\textbf{\emph{X}}}_i||_F^2=||{\textbf{\emph{D}}}_i({\textbf{\emph{B}}}_i-{\textbf{\emph{A}}}_i)||_F^2
=||{\textbf{\emph{U}}}_i{\rm diag}({\textbf{\emph{B}}}_i-{\textbf{\emph{A}}}_i){\textbf{\emph{V}}}_i||_F^2\\
&= {\rm Tr}({\textbf{\emph{U}}}_i{\rm diag}({\textbf{\emph{B}}}_i-{\textbf{\emph{A}}}_i){\textbf{\emph{V}}}_i
{\textbf{\emph{V}}}_{i}^T{\rm diag}({\textbf{\emph{B}}}_i-{\textbf{\emph{A}}}_i){\textbf{\emph{U}}}_{i}^T)\\
&= {\rm Tr}({\textbf{\emph{U}}}_i{\rm diag}({\textbf{\emph{B}}}_i-{\textbf{\emph{A}}}_i)
{\rm diag}({\textbf{\emph{B}}}_i-{\textbf{\emph{A}}}_i){\textbf{\emph{U}}}_{i}^T)\\
&={\rm Tr}({\rm diag}({\textbf{\emph{B}}}_i-{\textbf{\emph{A}}}_i)
{\textbf{\emph{U}}}_i{\textbf{\emph{U}}}_{i}^T{\rm diag}({\textbf{\emph{B}}}_i-{\textbf{\emph{A}}}_i))\\
&={\rm Tr}({\rm diag}({\textbf{\emph{B}}}_i-{\textbf{\emph{A}}}_i)
{\rm diag}({\textbf{\emph{B}}}_i-{\textbf{\emph{A}}}_i))\\
&=||{\textbf{\emph{B}}}_i-{\textbf{\emph{A}}}_i||_F^2,
\end{aligned}
\label{eq:28}
\end{equation} 
\end{proof}

Based on Lemma~\ref{lemma:1},  Lemma~\ref{lemma:3} and Theorem~\ref{theorem:1}, we have the following conclusion.

\begin{theorem}
\label{theorem:4}
Solving the NNM in Eq.~\eqref{eq:21.1} is equivalent to the $\ell_1$-norm minimization in Eq.~\eqref{eq:26} under the proposed adaptive dictionary.
\end{theorem}
\begin{proof}
Based on Lemma~\ref{lemma:3}, Eq.~\eqref{eq:26} can be rewritten as
 \begin{equation}
  \begin{aligned}
{\hat{\textbf{\emph{A}}}}_i&=\arg\min\limits_{{\textbf{\emph{A}}}_i}
\left(\frac{1}{2}||{\textbf{\emph{Y}}}_i-{\textbf{\emph{D}}}_i{\textbf{\emph{A}}}_i||_F^2+\lambda||{\textbf{\emph{A}}}_i||_1\right)\\
&=\arg\min\limits_{{\textbf{\emph{A}}}_i}\left(\frac{1}{2}||{\textbf{\emph{B}}}_i-{\textbf{\emph{A}}}_i||_F^2+\lambda||{\textbf{\emph{A}}}_i||_1\right)\\
&=\arg\min\limits_{\boldsymbol\alpha_i}\left(\frac{1}{2}||{\boldsymbol\beta}_i-{\boldsymbol\alpha}_i||_2^2+\lambda||{\boldsymbol\alpha}_i||_1\right),
\end{aligned}
\label{eq:29}
\end{equation} 
where ${\textbf{\emph{X}}}_i={{\textbf{\emph{D}}}_i{\textbf{\emph{A}}}_i}$ and ${\textbf{\emph{Y}}}_i={{\textbf{\emph{D}}}_i{\textbf{\emph{B}}}_i}$. ${{{{\boldsymbol\alpha}}}_i}$ and ${{{{\boldsymbol\beta}}}_i}$ denote the vectorization of the matrix ${{{\textbf{\emph{A}}}}_i}$ and ${{{\textbf{\emph{B}}}}_i}$, respectively.

Thus, based on Lemma~\ref{lemma:1}, we have
\begin{equation}
\begin{aligned}
{{\boldsymbol\alpha}}_i & ={\rm soft}({{\boldsymbol\beta}}_i,\lambda)= {\rm sgn}({{\boldsymbol\beta}}_i)\odot{\rm max}({\rm abs}({{\boldsymbol\beta}}_i)-\lambda,0).
\end{aligned}
\label{eq:30}
\end{equation}
Obviously, according to Eqs.~\eqref{eq:20} and ~\eqref{eq:21}, we have
\begin{equation}
\begin{aligned}
{\textbf{\emph{D}}}_i{\hat{\textbf{\emph{A}}}}_i&= \sum\nolimits_{j=1}^{n_1} {{\rm soft}({\boldsymbol\beta}}_{i,j},\lambda){\textbf{\emph{d}}}_{i,j}\\
&=\sum\nolimits_{j=1}^{n_1} {{\rm soft}({\boldsymbol\beta}}_{i,j},\lambda){\textbf{\emph{u}}}_{i,j}{\textbf{\emph{v}}}_{i,j}^T \\
&={\textbf{\emph{U}}}_i\mathcal{D}_{\lambda}(\boldsymbol\Sigma_i){\textbf{\emph{V}}}_i^T =  \mathcal{D}_{\lambda}({\textbf{\emph{Y}}}_i).
\end{aligned}
\label{eq:31}
\end{equation} 
where ${{\boldsymbol\beta}}_{i,j}$ represents the $j$-th element of the $i$-th group sparse coefficient ${{\boldsymbol\beta}}_{i}$, and $\boldsymbol\Sigma_i$ is the  singular value matrix of the $i$-th group ${\textbf{\emph{Y}}}_i$.

Obviously, based on Theorem~\ref{theorem:1}, we prove that NNM (Eq.~\eqref{eq:21.1}) is equivalent to the $\ell_1$-norm minimization based on GSR model (Eq.~\eqref{eq:26}).
\end{proof}

Similar to Theorem~\ref{theorem:4}, we have the following conclusion.
\begin{corollary}
\label{corollary:1}
Solving the WNNM in Eq.~\eqref{eq:22} is equivalent to the weighted $\ell_1$-norm minimization in Eq.~\eqref{eq:32} under the proposed adaptive dictionary.
\end{corollary}

\begin{proof}
The weighted $\ell_1$-norm minimization based on GSR can be represented as
\begin{equation}
\hat{\textbf{\emph{A}}}_i=\arg\min_{{\textbf{\emph{A}}}_i} \sum\nolimits_{i=1}^n\left(\frac{1}{2}||{\textbf{\emph{Y}}}_i-{\textbf{\emph{D}}}_i{\textbf{\emph{A}}}_i||_F^2+ ||{\emph{\textbf{{w}}}}_i{\textbf{\emph{A}}}_i||_1\right),
\label{eq:32}
\end{equation} 
where  ${\textbf{\emph{w}}}_i$ is a weight assigned to ${\textbf{\emph{A}}}_i$. The weight ${\textbf{\emph{w}}}_i$ will enhance the representation capability of group sparse coefficient ${\textbf{\emph{A}}}_i$.

For the weighted $\ell_1$-norm minimization in Eq.~\eqref{eq:32} and based on  Lemma~\ref{lemma:3}, we have,
\begin{equation}
\begin{aligned}
{\hat{\textbf{\emph{A}}}}_i&=\arg\min\limits_{{\textbf{\emph{A}}}_i}
\left(\frac{1}{2}||{\textbf{\emph{Y}}}_i-{\textbf{\emph{D}}}_i{\textbf{\emph{A}}}_i||_F^2 + ||{\emph{\textbf{{w}}}}_i{\textbf{\emph{A}}}_i||_1\right)\\
&=\arg\min\limits_{{\textbf{\emph{A}}}_i}\left(\frac{1}{2}||{\textbf{\emph{B}}}_i-{\textbf{\emph{A}}}_i||_F^2 + ||{\emph{\textbf{{w}}}}_i {\textbf{\emph{A}}}_i||_1\right)\\
&=\arg\min\limits_{\boldsymbol\alpha_i}\left(\frac{1}{2}||{\boldsymbol\beta}_i-{\boldsymbol\alpha}_i||_2^2+||w_i{\boldsymbol\alpha}_i||_1\right),
\end{aligned}
\label{eq:33}
\end{equation} 
where ${\textbf{\emph{X}}}_i={{\textbf{\emph{D}}}_i{\textbf{\emph{A}}}_i}$ and ${\textbf{\emph{Y}}}_i={{\textbf{\emph{D}}}_i{\textbf{\emph{B}}}_i}$. ${{{{\boldsymbol\alpha}}}_i}$, ${{{{\boldsymbol\beta}}}_i}$ and ${{w}}_i$ denote the vectorization of the matrix ${{{\textbf{\emph{A}}}}_i}$, ${{{\textbf{\emph{B}}}}_i}$ and $\emph{\textbf{{w}}}_i$, respectively.

Therefore, based on Lemma~\ref{lemma:2}, we have
\begin{equation}
{{\boldsymbol\alpha}}_i  ={\rm soft}({{\boldsymbol\beta}}_i,{{{w}}}_i)=  {\rm max}({{\boldsymbol\beta}}_i-{{{w}}}_i,0)
\label{eq:34}
\end{equation}

Based on  Eqs.~\eqref{eq:20} and ~\eqref{eq:21}, we have
\begin{equation}
\begin{aligned}
&\hat{\textbf{\emph{X}}}_i={{\textbf{\emph{D}}}_i{\textbf{\emph{A}}}_i}= \sum\nolimits_{j=1}^{n_0} {{\rm soft}({\boldsymbol\beta}}_{i,j},{{{w}}}_{i,j}){\textbf{\emph{d}}}_{i,j}\\
&=\sum\nolimits_{j=1}^{n_0} {{\rm soft}({\boldsymbol\beta}}_{i,j},{{{w}}}_{i,j}){\textbf{\emph{u}}}_{i,j}{\textbf{\emph{v}}}_{i,j}^T\\ &={\textbf{\emph{U}}}_i\mathcal{D}_{\emph{\textbf{{w}}}_i}(\boldsymbol\Sigma_i){\textbf{\emph{V}}}_i^T = \mathcal{D}_{{\textbf{\emph{w}}}_i} (\textbf{\emph{Y}}_i).
\end{aligned}
\label{eq:35}
\end{equation} 

Obviously, based on Theorem~\ref{theorem:3}, we prove that WNNM (Eq.~\eqref{eq:22}) is equivalent to the weighted $\ell_1$-norm minimization based on GSR model (Eq.~\eqref{eq:32}).
\end{proof}

Therefore, we prove that NNM and WNNM is equivalent to the $\ell_1$-norm minimization based on GSR and the weighted $\ell_1$-norm minimization based on GSR, respectively.

Note that, although the unit of the GSR model is each group, it conducts the sparsity of each column of each group. Therefore, Proposition~\ref{proposition:1} is also suit for the GSR model. Based on Theorem~\ref{theorem:4} and Corollary~\ref{corollary:1}, namely, NNM and WNNM are equivalent to the $\ell_1$-norm minimization based on GSR and the weighted $\ell_1$-norm minimization based on GSR, respectively, and in terms of Proposition~\ref{proposition:1}, we have the following conclusion.

\begin{proposition}
\label{proposition:2}
WNNM is more feasible than NNM, \ie,
\begin{equation}
\min_{\textbf{\emph{X}}} ||\textbf{\emph{X}}||_{\rm {\textbf{w}},*}  \gg  \min_{\textbf{\emph{X}}} ||\textbf{\emph{X}}||_*
\label{eq:36}
\end{equation} 
where ${ v_1}\gg { v_2}$ denotes that the entry ${ v_1}$ is more than the corresponding entry of ${ v_2}$.
\end{proposition}

\begin{proof}
Based on Theorem~\ref{theorem:4} and Corollary~\ref{corollary:1}, we prove that NNM and WNNM are equivalent to the $\ell_1$-norm minimization based on GSR and the weighted $\ell_1$-norm minimization based on GSR, respectively. Due to the fact that Proposition~\ref{proposition:1} is suit for the GSR model, and thus based on Proposition~\ref{proposition:1}, \ie, the weighted $\ell_1$-norm minimization can obtain more sparsity performance than traditional $\ell_1$-norm minimization, we prove that WNNM is more feasible than NNM.
\end{proof}

It is worth noting that the dictionary can be learned in various manners and the proposed adaptive dictionary learning approach is just one of them. Although the designed adaptive dictionary learning seems to translate the sparse representation into the rank minimization problem directly, the main difference between sparse representation and the rank minimization models is that sparse representation has a dictionary learning process while the rank minimization problem does not, to the best of our knowledge.

According to the above analysis, we present a scheme to analyze WNNM and NNM from the perspective of the GSR. We design an adaptive dictionary to bridge the gap between the GSR and the rank minimization models. Based on this scheme, we provide a mathematical derivation to explain why WNNM is more feasible than NNM.

\section {WNNM Model for Image Restoration}
\label{sec:4}
Since image restoration (IR) is an ideal test bed bench to measure different statistical image models, in this section, we employ the WNNM model on two image restoration tasks, \ie, image denoising and image inpainting. Specifically, image restoration aims to reconstruct a high quality image $\textbf{\emph{x}}$ from its degraded observation $\textbf{\emph{y}}$, which can be generally expressed as
\begin{equation}
\textbf{\emph{y}}= \textbf{\emph{H}}\textbf{\emph{x}} + \boldsymbol \eta,
\label{eq:37}
\end{equation} 
where $\textbf{\emph{H}}$ is a non-invertible linear degradation operator and $\boldsymbol\eta$ is the vector of some independent white Gaussian noise. With different settings of matrix $\textbf{\emph{H}}$, various IR problems can be derived from Eq.~\eqref{eq:37}, such as image denoising \cite{39} when $\textbf{\emph{H}}$ is an identity matrix; image deblurring \cite{43} when $\textbf{\emph{H}}$ is a blur operator; image inpainting \cite{44} when $\textbf{\emph{H}}$ is a mask and image compressive sensing (CS) recovery when $\textbf{\emph{H}}$ is a random projection matrix \cite{29}. In this work, we focus on the image denoising and image inpainting problems.

Due to the ill-posed nature of IR, it is critical to exploit the prior knowledge that characterizes the statistical features of the images. Motivated by the fact that image patches that have similar patterns can be spatially far from each other and thus can be collected in the whole image. The well-known nonlocal self-similarity (NSS) prior \cite{45}, which characterizes the repetitiveness of textures and structures by natural images within nonlocal regions, implies that many similar patches can be searched for any exemplar patch. To be concrete, in the stage of image restoration, each degraded image patch $\textbf{\emph{y}}_i$ is extracted from the degraded image $\textbf{\emph{y}}$. Like in the subsection~\ref{sec:2.4}, we search for its $k$ similar patches to generate a group $\textbf{\emph{Y}}_i$, \ie, ${\textbf{\emph{Y}}}_i=\{{\textbf{\emph{y}}}_{i,1}, {\textbf{\emph{y}}}_{i,2},...,{\textbf{\emph{y}}}_{i,k}\}$. Then we have ${\textbf{\emph{Y}}}_i = \textbf{\emph{X}}_i +{\textbf{\emph{N}}}_i$, where $\textbf{\emph{X}}_i$ and ${\textbf{\emph{N}}}_i$ are the group matrices of original image and noise, respectively. Since all the patches in each data matrix have similar structures, the constructed data matrix ${\textbf{\emph{Y}}}_i$ has a low rank property and we can use the low rank approximation method to estimate $\textbf{\emph{X}}_i$ from $\textbf{\emph{Y}}_i$. Following this, by aggregating all the recovered groups, the whole image can be obtained. In this section, we apply the WNNM model to $\textbf{\emph{Y}}_i$ to estimate $\textbf{\emph{X}}_i$ for IR, and $\textbf{\emph{X}}_i$ can be reconstructed by solving the following optimization problem,
\begin{equation}
\hat{\textbf{\emph{X}}}_i=\arg\min_{\textbf{\emph{X}}_i}\left(\frac{1}{2}||\textbf{\emph{Y}}_i-\textbf{\emph{X}}_i||_F^2 + ||\textbf{\emph{X}}_i||_{ {\emph{\textbf{w}}}_i,*}\right),
\label{eq:38}
\end{equation} 

To obtain an effective solution of Eq.~\eqref{eq:38}, we introduce the following theorem.

\begin{theorem}
\label{theorem:5}
Let $\textbf{\emph{Y}}_i= \textbf{\emph{U}}_i\boldsymbol\Delta_i\textbf{\emph{V}}_i^T$ be the SVD of $\textbf{\emph{Y}}_i\in\Re^{m\times k}$ and $\boldsymbol\Delta_i=diag(\delta_{i,1}, ..., \delta_{i,n_0})$, $n_0= min(m, k)$. The optimal solution $\textbf{\emph{X}}_i$ to the problem Eq.~\eqref{eq:38} is $\textbf{\emph{U}}_i\boldsymbol\Sigma_i\textbf{\emph{V}}_{i}^T$, where $\boldsymbol\Sigma_{i}=diag(\sigma_{i,1}, ..., \sigma_{i,n_0})$. Then the optimal solution of the $j$-th diagonal element $\sigma_{i,j}$ of the diagonal matrix $\boldsymbol\Sigma_{i}$ is solved by the following problem,
\begin{equation}
\begin{aligned}
&\min\limits_{\sigma_{i,j}\geq0}
\left(\frac{1}{2}(\delta_{i, j}-\sigma_{i, j})^2+ w_{i, j} \sigma_{i,j} \right),
\end{aligned}
\label{eq:39}
\end{equation}
where $\sigma_{i,j}$ represents the $j$-th singular value of each data matrix $\textbf{\emph{X}}_i$.
\end{theorem}

\begin{proof}
See \cite{16}.
\end{proof}

Therefore, the minimization problem of Eq.~\eqref{eq:38} can be simplified by minimizing the problem of Eq.~\eqref{eq:39}. For fixed $\delta_{i, j}$ and $w_{i, j}$ and based on Lemma~\ref{lemma:2}, the closed-form solution of Eq.~\eqref{eq:39} can be expressed as
\begin{equation}
\sigma_{i, j}= {\rm soft} (\delta_{i, j}, w_{i, j}) = {\rm max} (\delta_{i, j} - w_{i, j}, 0).
\label{eq:40}
\end{equation}

With the solution of $\boldsymbol\Sigma_i$ in Eq.~\eqref{eq:40}, the clean group matrix ${{\textbf{\emph{X}}}_i}$ can be recovered as ${{\hat{\textbf{\emph{X}}}}_i}={{\textbf{\emph{U}}}_i}\boldsymbol\Sigma_i{{\textbf{\emph{V}}}_i}$. Then the latent image ${{\hat{\textbf{\emph{x}}}}}$ can be reconstructed by aggregating all the group matrices $\{{{\textbf{\emph{X}}}_i}\}$.

Inspired by the singular values have clear physical meanings, for the weight $\textbf{\emph{w}}_i$ of each group $\textbf{\emph{X}}_i$, large singular values of each group $\textbf{\emph{X}}_i$ usually present major edge and texture information, and vice versa. Therefore, we usually shrink large singular values less, while shrinking smaller ones more \cite{15,32}.  In other words, the weight $\textbf{\emph{w}}_i$ of each group $\textbf{\emph{X}}_i$ is set to be inverse to the singular values, and thus, in \cite{15}, the weight is heuristically set as $w_{i, j} = c/(\sigma_{i, j}+\varepsilon)$, where $c$ and $\varepsilon$ are the small constant. However, WNNM  sometimes pops out error in the operation of SVD, owe to this weight setting. In this subsection, we present an adaptive weight setting scheme to avoid this error. Specifically, inspired by \cite{46,47}, the weight $\textbf{\emph{w}}_i$ of each group $\hat{\textbf{\emph{X}}}_i$ is set as $\textbf{\emph{w}}_i = [w_{i,1}, w_{i,2}, ..., w_{i, j}]$, and we have
\begin{equation}
\textbf{\emph{w}}_i=\frac{c* 2\sqrt{2}{\sigma_n}^2}{(\boldsymbol\gamma_i+\varepsilon)}
\label{eq:41}
\end{equation}
where $\sigma_n$ represents the additive white Gaussian noise and $\boldsymbol\gamma_i$ denotes the estimated standard variance of the singular values of each group $\hat{\textbf{\emph{X}}}_i$. About the robustness analysis of Eq.~\eqref{eq:41}, more details please see \cite{46}. Also, we did massive  experiments and found that the proposed scheme cannot lead to the error in the operation of SVD.

Throughout the numerical experiments, we choose the following stoping iteration for the proposed image restoration algorithm, i.e,

\begin{equation}
\frac{||\hat{\textbf{\emph{x}}}^t-\hat{\textbf{\emph{x}}}^{t-1}||_2^2}{||\hat{\textbf{\emph{x}}}^{t-1}||_2^2} <\tau
\label{eq:42}
\end{equation}
where $\tau$ is a small constant. The complete description of the WNNM model for image restoration is shown in Algorithm~\ref{algo:1}.

\begin{center}
	\begin{algorithm}[htbp]
		\caption{WNNM model for image restoration.}
		\begin{algorithmic}[1]
			\REQUIRE The observed image $\textbf{\emph{y}}$ and the measurement matrix $\textbf{\emph{H}}$.
			\STATE  Initialize $\hat{\textbf{\emph{x}}} = \textbf{\emph{y}}$.
			\FOR{$k=0$ \TO Max-Iter }
			\IF {$\textbf{\emph{H}}$ is identity matrix }
			\STATE  Iterative regularization $\textbf{\emph{y}}^{t}=\hat{\textbf{\emph{x}}}^{t-1}+ \mu (\textbf{\emph{y}}-\hat{\textbf{\emph{x}}}^{t-1})$;
			\ELSIF{$\textbf{\emph{H}}$ is mask operator}
			\STATE  Iterative regularization $\textbf{\emph{y}}^{t}=\hat{\textbf{\emph{x}}}^{t-1}+ \mu \textbf{\emph{H}}^T (\textbf{\emph{y}}-\textbf{\emph{H}}\textbf{\emph{x}}^{t-1})$;

			\ENDIF
			\FOR{Each patch ${\textbf{\emph{y}}}_{i}$ in ${\textbf{\emph{y}}}^t$ }	
			\STATE Find nonlocal similar patches to form a group ${\textbf{\emph{Y}}}_{i}$;
			\STATE Singular value decomposition $[\textbf{\emph{U}}_i, \boldsymbol\Delta_i, \textbf{\emph{V}}_i]= SVD ({\textbf{\emph{Y}}}_i)$;
			\STATE Estimate the weight $\textbf{\emph{w}}_i$ of each group by computing  Eq.~\eqref{eq:41};
			\STATE Calculate $\boldsymbol\Sigma_i$ by  Eq.~\eqref{eq:40};
			\STATE Get the estimation:   $\hat{\textbf{\emph{X}}}_i=\textbf{\emph{U}}_i\boldsymbol\Sigma_i\textbf{\emph{V}}_i^T$;
			\ENDFOR
			\STATE Aggregate ${\textbf{\emph{X}}}_i$  to form the restored image $\hat{\textbf{\emph{x}}}^{t}$.
			\ENDFOR
			\STATE $\textbf{Output:}$ The final restored image $\hat{\textbf{\emph{x}}}$.
		\end{algorithmic}
		\label{algo:1}
	\end{algorithm}
\end{center}

\begin{figure}[!htbp]
	\begin{minipage}[b]{1\linewidth}
		\centering
		\centerline{\includegraphics[width=9cm]{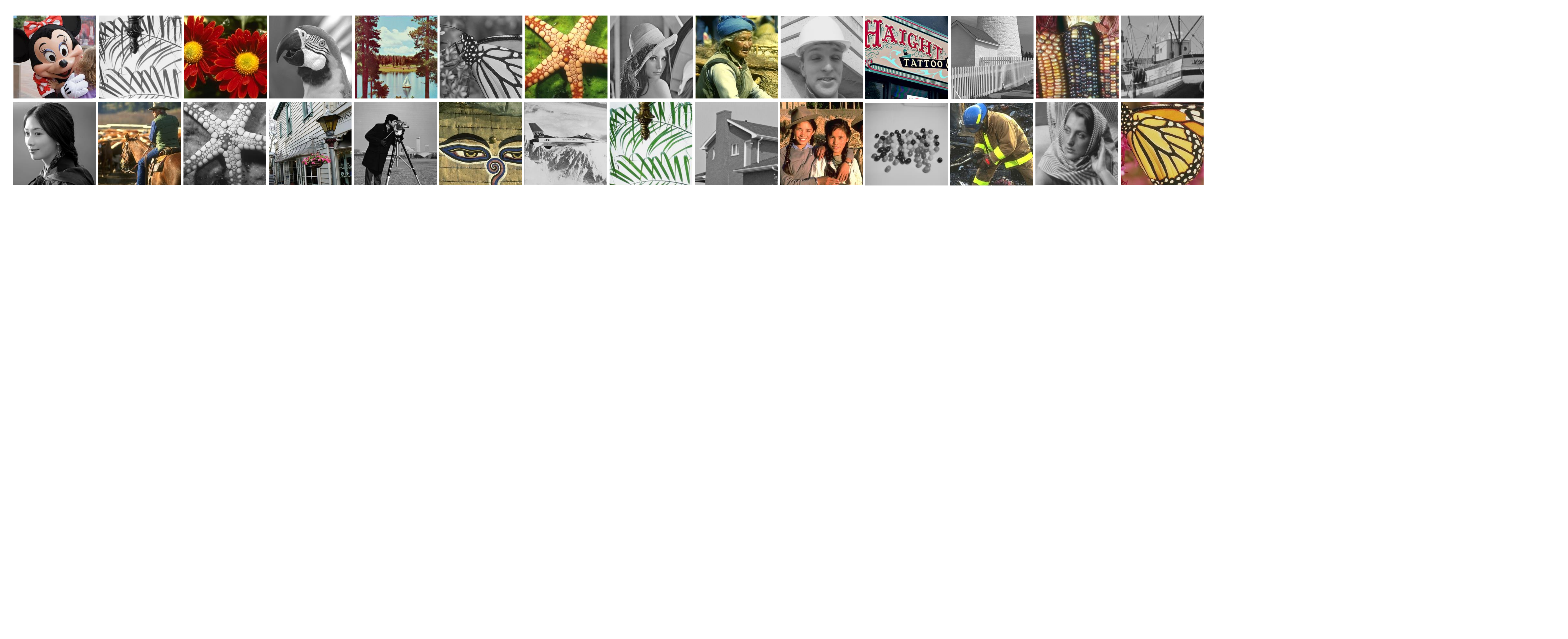}}
	\end{minipage}
	\vspace{-3mm}
	\caption{The test images for the experiments.  Top row, from left to right: Mickey, Leaves, Flower, Parrot, Lake, Monarch, Starfish, Lena, Nanna, Foreman, Haight, Fence, Corn, boats. Bottom row, from left to right: Lin, Cowboy, Starfish, Light, C.man, Mural, Airplane, Leaves, House, Girl, J. Bean, Fireman, Barbara, Butterfly.}
	\label{fig:2}
\vspace{-3mm}
\end{figure}

\section {Experimental Results}
\label{sec:5}
In this section, we conduct a variety of experiments in the applications of image denoising and image inpainting. Since the GSR is exploited to analyze that WNNM is more feasible than NNM, we called the proposed scheme as GSR-WNNM. To verify the effectiveness of the GSR-WNNM, we have implemented a variant of the GSR that use NNM, denoted as GSR-NNM. All the experimental images are shown in  Fig.~\ref{fig:2}. To evaluated the quality of the recovered images, both PSNR and SSIM \cite{48} metrics are used.  The Matlab code can be downloaded at: \url{https://drive.google.com/open?id=1BdRxJAf798KzURnErYqeDZ9qSVPoUNYC}.

\subsection {Image Denoising}
In image denoising, to validate the denoising performance of the proposed GSR-WNNM, we compare it with leading denoising methods, including BM3D \cite{39}, EPLL \cite{49}, NCSR \cite{50}, Plow \cite{51}, PGPD \cite{52}, OGLR \cite{53} and GSR-NNM methods. The parameter setting of the proposed GSR-WNNM is as follows. The size of each patch $\sqrt{m} \times \sqrt{m}$ is set be 6$\times$6, 7$\times$7, 8$\times$8 and 9$\times$9 for $\sigma_n\leq20$, $20<\sigma_n\leq40$, $40<\sigma_n\leq75$ and $75<\sigma_n\leq100$, respectively. The number of the similar patches $k$ is set to be 60, 70, 80 and 100 for  $\sigma_n\leq40$, $40<\sigma_n\leq50$, $50<\sigma_n\leq 75$ and $75<\sigma_n\leq100$, respectively.  $\tau$ and $c$ are set to be (0.0013, 0.65), (0.001, 0.75), (0.0012, 0.65), (0.0013, 0.65), (0.0017, 0.55) and (0.0019, 0.60) for $\sigma_n\leq 20$, $20<\sigma_n\leq30$, $30<\sigma_n\leq40$, $40<\sigma_n\leq50$, $50<\sigma_n\leq75$ and $75<\sigma_n\leq100$, respectively. The search window for similar patches is set $L = 30$ and $\varepsilon = 10^{-16}$.

We present the denoising results on six noise levels, \ie, $\sigma_n$ = 20, 30, 40, 50, 75 and 100. The PSNR and SSIM results under these noise levels for all competing denoising methods for the 14 widely used test images are shown in Table~\ref{Tab:1} and Table~\ref{Tab:2}, respectively (the highest PSNR and SSIM values are marked in bold). It is clearly seen that the proposed GSR-WNNM significantly outperforms the GSR-NNM method, that is, WNNM model is more feasible than NNM model. In a majority of cases, on can observe that the proposed GSR-WNNM can obtain better PSNR and SSIM results than other competing methods. The average PSNR gains of the proposed GSR-WNNM over BM3D, EPLL, NCSR, Plow, PGPD, OGLR and GSR-NNM methods are as much as 0.32dB, 1.05dB, 0.44dB, 0.85dB, 0.21dB, 0.55dB and 1.50dB, respectively. In terms of SSIM, it can be seen that the proposed GSR-WNNM can also achieve higher results than other competing methods. The only exception is when $\sigma_n$ = 20 for which BM3D is slightly higher than the proposed GSR-WNNM. Nonetheless, under high noise level $\sigma_n$ =100, the proposed GSR-WNNM consistently outperforms other competing methods for all cases. The visual quality comparisons in the case of $\sigma_n$ =100 for test images $\emph{House}$, $\emph{Leaves}$ and $\emph{Monarch}$ are shown in Fig.~\ref{fig:3}, Fig.~\ref{fig:4} and Fig.~\ref{fig:5}, respectively. It can be seen that the over-smooth phenomena or undesirable artifacts are generated by BM3D, EPLL, NCSR, Plow, PGPD, OGLR and GSR-NNM methods. In contrast, the proposed GSR-WNNM not only removes most of the artifacts, but also provides better denoising performance than BM3D, EPLL, NCSR, Plow, PGPD, OGLR and GSR-NNM methods on both edges and textures. Therefore, these results validate the efficient of the proposed GSR-WNNM and also demonstrate that WNNM model is more feasible than NNM model.

\begin{figure}[!htbp]
\vspace{-3mm}
	\centerline{\includegraphics[width=9cm]{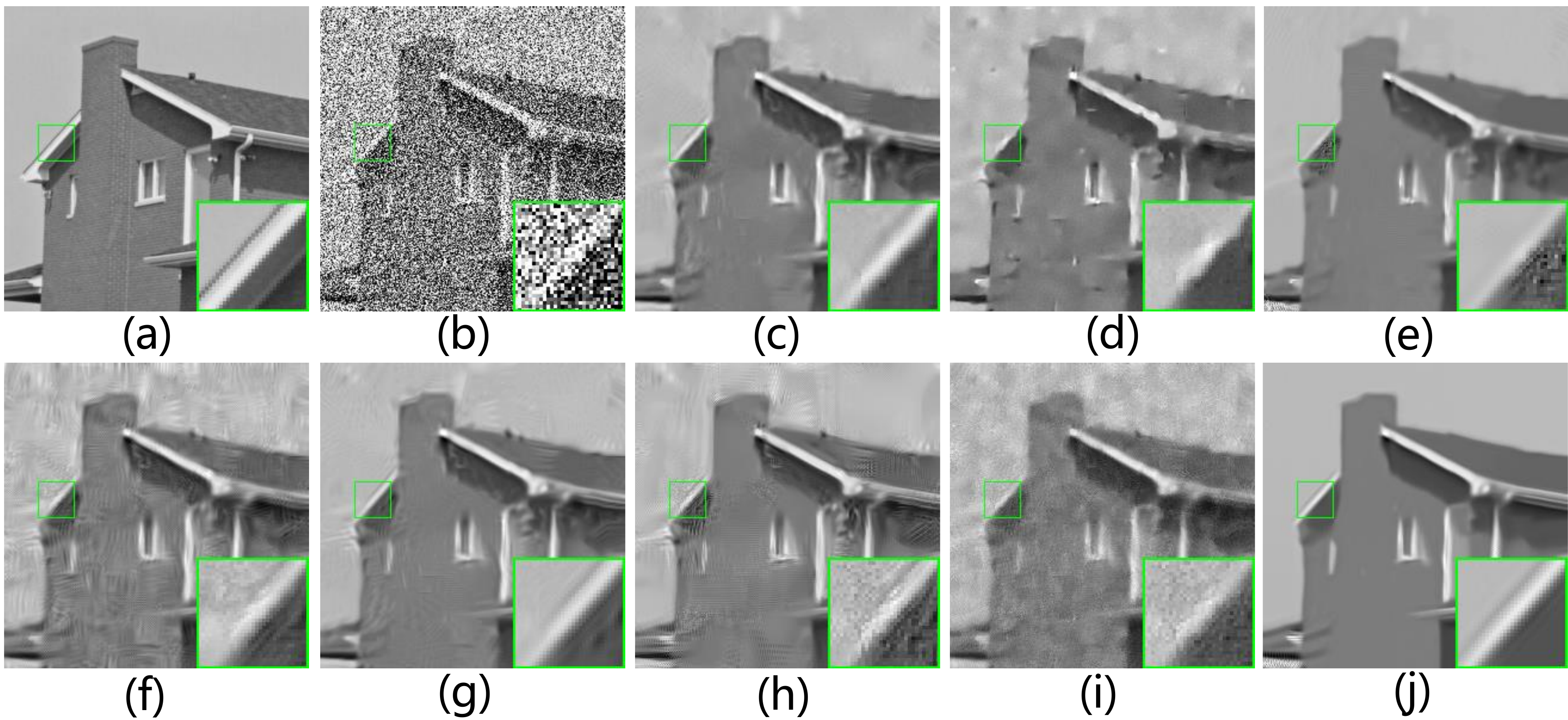}}
	\vspace{-3mm}
	\caption{Denoising performance comparison of image $\emph{House}$ with $\sigma_n$ = 100. (a) Original image; (b) Noisy image; (c) BM3D \cite{39} (PSNR = 25.87dB, SSIM = 0.7203); (d) EPLL \cite{49} (PSNR = 25.21dB, SSIM = 0.6695); (e) NCSR \cite{50} (PSNR = 25.49dB, SSIM = 0.7397); (f) Plow \cite{51} (PSNR = 24.72dB, SSIM = 0.5874); (g) PGPD \cite{52} (PSNR = 26.17dB, SSIM = 0.7195);  (h) OGLR \cite{53} (PSNR = 25.07dB, SSIM = 0.6373);   (i) GSR-NNM (PSNR = 23.66dB, SSIM = 0.4918); (j) GSR-WNNM (PSNR = \textbf{26.71dB}, SSIM = \textbf{0.7756}).}
	\label{fig:3}
	\vspace{-3mm}
\end{figure}

\begin{figure}[!htbp]
\vspace{-3mm}
	\centerline{\includegraphics[width=9cm]{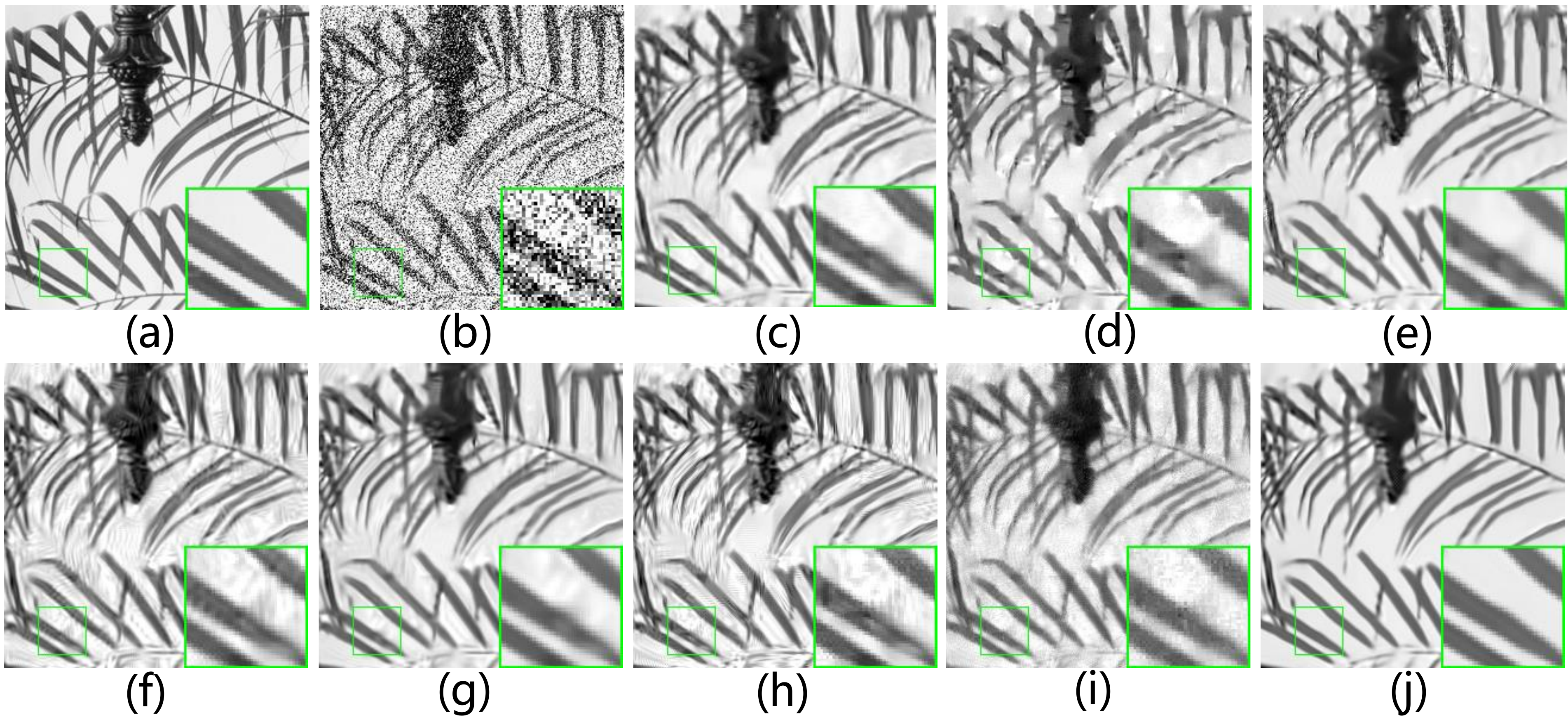}}
	\vspace{-3mm}
	\caption{Denoising performance comparison of image $\emph{Leaves}$ with $\sigma_n$ = 100. (a) Original image; (b) Noisy image; (c) BM3D \cite{39} (PSNR = 20.90dB, SSIM = 0.7482); (d) EPLL \cite{49} (PSNR = 20.26dB, SSIM = 0.7163); (e) NCSR \cite{50} (PSNR = 20.84dB, SSIM = 0.7622); (f) Plow \cite{51} (PSNR = 20.43dB, SSIM = 0.6814); (g) PGPD \cite{52} (PSNR = 20.95dB, SSIM = 0.7469);  (h) OGLR \cite{53} (PSNR = 20.28dB, SSIM = 0.6827);   (i) GSR-NNM (PSNR = 19.57dB, SSIM = 0.6345); (j) GSR-WNNM (PSNR = \textbf{21.56dB}, SSIM = \textbf{0.7964}).}
	\label{fig:4}
	\vspace{-3mm}
\end{figure}

\begin{figure}[!htbp]
\vspace{-3mm}
	\centerline{\includegraphics[width=9cm]{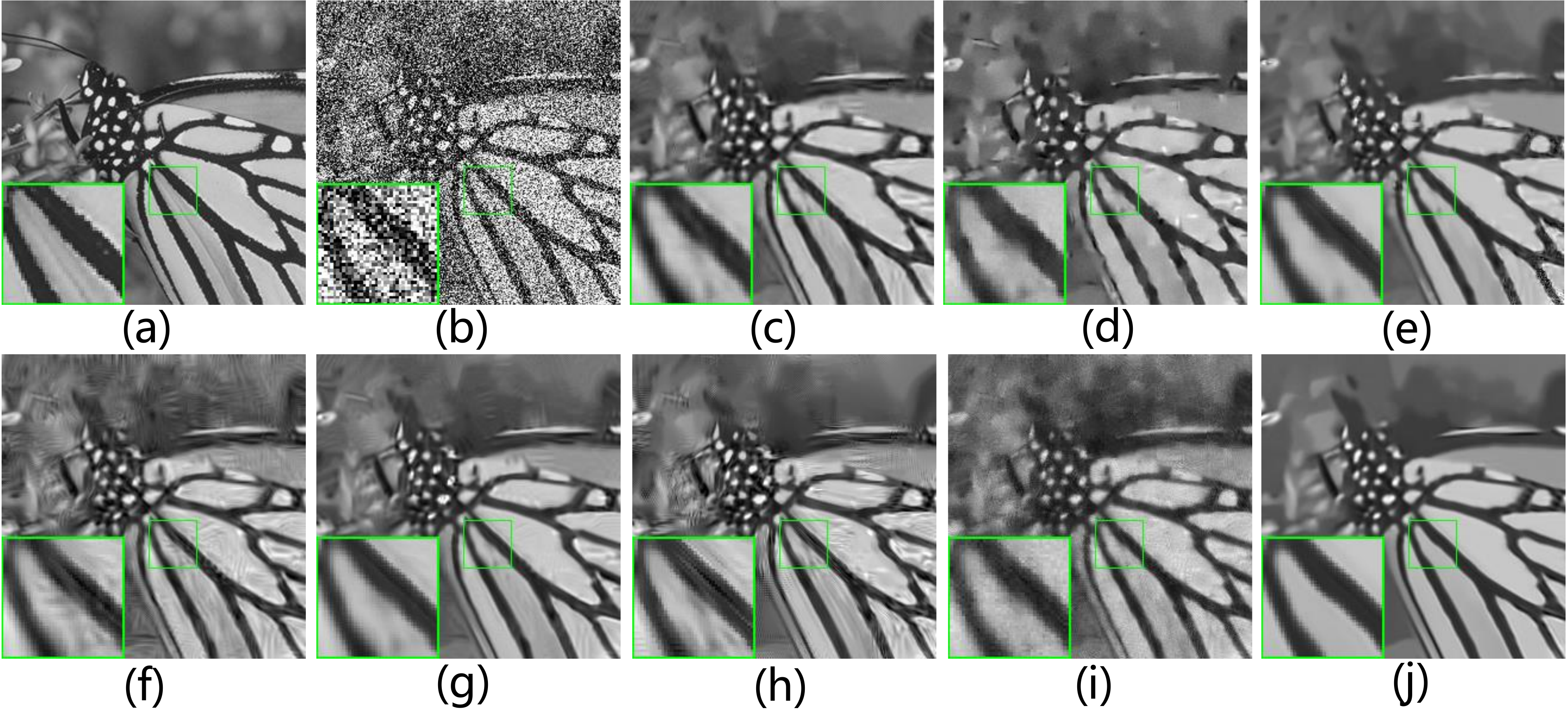}}
	\vspace{-3mm}
	\caption{Denoising performance comparison of image $\emph{Monarch}$ with $\sigma_n$ = 100. (a) Original image; (b) Noisy image; (c) BM3D \cite{39} (PSNR = 22.52dB, SSIM = 0.7021); (d) EPLL \cite{49} (PSNR = 22.24dB, SSIM = 0.6771); (e) NCSR \cite{50} (PSNR = 22.10dB, SSIM = 0.7109); (f) Plow \cite{51} (PSNR = 21.83dB, SSIM = 0.6102); (g) PGPD \cite{52} (PSNR = 22.56dB, SSIM = 0.7029);  (h) OGLR \cite{53} (PSNR = 21.87dB, SSIM = 0.6419);   (i) GSR-NNM (PSNR = 21.03dB, SSIM = 0.5596); (j) GSR-WNNM (PSNR = \textbf{22.86dB}, SSIM = \textbf{0.7307}).}
	\label{fig:5}
	\vspace{-3mm}
\end{figure}

	\begin{table*}[!htbp]
			\caption{PSNR ($\textnormal{d}$B) comparison of BM3D \cite{39}, EPLL \cite{49}, NCSR \cite{50}, Plow \cite{51}, PGPD \cite{52}, OGLR \cite{53}, GSR-NNM and GSR-WNNM for image denoising.}
		\centering
		\resizebox{1.00\textwidth}{!}
		{
			\begin{tabular}{|c|c|c|c|c|c|c|c|c||c|c|c|c|c|c|c|c|c|c|c|c|c|c|c|c|}
				\hline
				\multicolumn{1}{|c|}{}&\multicolumn{8}{|c||}{$\sigma_n=20$}&\multicolumn{8}{|c|}{$\sigma_n=30$}\\
				\hline
				\multirow{2}{*}{\textbf{{Images}}}&\multirow{2}{*}{\textbf{{BM3D}}}
				&\multirow{2}{*}{\textbf{{EPLL}}}&\multirow{2}{*}{\textbf{{NCSR}}}&\multirow{2}{*}{\textbf{{Plow}}}&\multirow{2}{*}{\textbf{{PGPD}}}
				&\multirow{2}{*}{\textbf{{OGLR}}}&{\textbf{{GSR-}}}&{\textbf{{GSR-}}}&\multirow{2}{*}{\textbf{{BM3D}}}&\multirow{2}{*}{\textbf{{EPLL}}}
                &\multirow{2}{*}{\textbf{{NCSR}}}&\multirow{2}{*}{\textbf{{Plow}}}&\multirow{2}{*}{\textbf{{PGPD}}}&\multirow{2}{*}{\textbf{{OGLR}}}
                &{\textbf{{GSR-}}}&{\textbf{{GSR-}}}\\
				& &  & & & & &{\textbf{NNM}} &{\textbf{WNNM}} & & & & & & & {\textbf{NNM}}& {\textbf{WNNM}}  \\
				\hline
				\multirow{1}{*}{Airplane}
&	30.59 	&	30.60 	&	30.50 	&	29.98 	&	30.80 	&	30.17 	&	29.01 	&	\textbf{30.87} 	&	28.49 	&	28.54 	&	28.34 	&	28.03 	&	28.63 	&	28.21 	&	27.62 	&	\textbf{28.67}
\\
\hline
				\multirow{1}{*}{Barbara}
&	31.24 	&	29.85 	&	31.10 	&	30.75 	&	31.12 	&	30.89 	&	30.07 	&	\textbf{31.53} 	&	29.08 	&	27.58 	&	28.68 	&	28.99 	&	28.93 	&	28.84 	&	28.08 	&	\textbf{29.36}
\\
\hline
				\multirow{1}{*}{boats}
&	31.42 	&	30.87 	&	31.26 	&	30.90 	&	31.38 	&	31.20 	&	29.96 	&	\textbf{31.51} 	&	29.33 	&	28.85 	&	29.04 	&	29.01 	&	29.32 	&	29.11 	&	28.08 	&	\textbf{29.34}
\\
\hline
				\multirow{1}{*}{C. man}
&	\textbf{30.49} 	&	30.34 	&	30.41 	&	29.55 	&	30.36 	&	30.14 	&	28.79 	&	30.27 	&	\textbf{28.64} 	&	28.35 	&	28.52 	&	27.80 	&	28.54 	&	28.26 	&	27.67 	&	28.51
\\
\hline
				\multirow{1}{*}{Fence}
&	29.93 	&	29.24 	&	30.05 	&	29.13 	&	29.99 	&	29.82 	&	28.86 	&	\textbf{30.14} 	&	28.19 	&	27.22 	&	28.13 	&	27.59 	&	28.13 	&	28.12 	&	27.43 	&	\textbf{28.34}
\\
\hline
				\multirow{1}{*}{Foreman}
&	34.54 	&	33.67 	&	34.42 	&	34.21 	&	34.44 	&	34.50 	&	33.34 	&	\textbf{34.67} 	&	32.75 	&	31.70 	&	32.61 	&	32.45 	&	32.83 	&	32.84 	&	30.24 	&	\textbf{33.27}
\\
\hline
				\multirow{1}{*}{House}
&	33.77 	&	32.99 	&	33.81 	&	33.40 	&	33.85 	&	33.77 	&	32.30 	&	\textbf{33.89} 	&	32.09 	&	31.24 	&	32.01 	&	31.67 	&	32.24 	&	32.02 	&	29.85 	&	\textbf{32.52}\\
\hline
				\multirow{1}{*}{J. Bean}
&	34.18 	&	33.79 	&	34.37 	&	33.80 	&	34.28 	&	34.44 	&	32.61 	&	\textbf{34.67} 	&	31.97 	&	31.55 	&	31.99 	&	31.61 	&	31.99 	&	32.15 	&	29.77 	&	\textbf{32.51}
\\
\hline
				\multirow{1}{*}{Leaves}
&	30.09 	&	29.40 	&	30.34 	&	29.08 	&	30.46 	&	29.87 	&	28.93 	&	\textbf{30.99} 	&	27.81 	&	27.19 	&	28.04 	&	27.00 	&	27.99 	&	27.77 	&	27.17 	&	\textbf{28.53}
\\
\hline
				\multirow{1}{*}{Lena}
&	31.52 	&	31.25 	&	31.48 	&	30.98 	&	31.64 	&	31.27 	&	30.03 	&	\textbf{31.68} 	&	29.46 	&	29.18 	&	29.32 	&	29.16 	&	29.60 	&	29.36 	&	28.29 	&	\textbf{29.60}\\
\hline
				\multirow{1}{*}{Lin}
&	\textbf{32.83} 	&	32.62 	&	32.66 	&	32.45 	&	32.79 	&	32.77 	&	31.17 	&	32.78 	&	\textbf{30.95} 	&	30.67 	&	30.65 	&	30.76 	&	30.96 	&	30.85 	&	29.22 	&	30.90
\\
\hline
				\multirow{1}{*}{Monarch}
&	30.35 	&	30.49 	&	30.52 	&	29.50 	&	30.68 	&	30.13 	&	29.47 	&	\textbf{30.97} 	&	28.36 	&	28.36 	&	28.38 	&	27.77 	&	28.49 	&	28.33 	&	27.63 	&	\textbf{28.77}\\
\hline
				\multirow{1}{*}{Parrot}
&	32.32 	&	32.00 	&	32.25 	&	31.85 	&	32.31 	&	32.12 	&	30.95 	&	\textbf{32.38} 	&	30.33 	&	30.00 	&	30.20 	&	29.88 	&	30.30 	&	30.24 	&	28.97 	&	\textbf{30.48}\\
\hline
				\multirow{1}{*}{Starfish}
&	29.67 	&	29.58 	&	29.85 	&	28.83 	&	29.84 	&	29.46 	&	28.63 	&	\textbf{30.11} 	&	27.65 	&	27.52 	&	27.69 	&	27.02 	&	27.67 	&	27.47 	&	27.10 	&	\textbf{27.94}
\\
\hline
				\multirow{1}{*}{\textbf{Average}}
&	31.64 	&	31.19 	&	31.65 	&	31.03 	&	31.71 	&	31.47 	&	30.29 	&	\textbf{31.89} 	&	29.65 	&	29.14 	&	29.54 	&	29.19 	&	29.69 	&	29.54 	&	28.37 	&	\textbf{29.91}
\\
\hline
				\multicolumn{1}{|c|}{}&\multicolumn{8}{|c||}{$\sigma_n=40$}&\multicolumn{8}{|c|}{$\sigma_n=50$}\\
				\hline
				\multirow{2}{*}{\textbf{{Images}}}&\multirow{2}{*}{\textbf{{BM3D}}}
				&\multirow{2}{*}{\textbf{{EPLL}}}&\multirow{2}{*}{\textbf{{NCSR}}}&\multirow{2}{*}{\textbf{{Plow}}}&\multirow{2}{*}{\textbf{{PGPD}}}
				&\multirow{2}{*}{\textbf{{OGLR}}}&{\textbf{{GSR-}}}&{\textbf{{GSR-}}}&\multirow{2}{*}{\textbf{{BM3D}}}&\multirow{2}{*}{\textbf{{EPLL}}}
                &\multirow{2}{*}{\textbf{{NCSR}}}&\multirow{2}{*}{\textbf{{Plow}}}&\multirow{2}{*}{\textbf{{PGPD}}}&\multirow{2}{*}{\textbf{{OGLR}}}
                &{\textbf{{GSR-}}}&{\textbf{{GSR-}}}\\
				& &  & & & & &{\textbf{NNM}} &{\textbf{WNNM}} & & & & & & & {\textbf{NNM}}& {\textbf{WNNM}}  \\
				\hline
				\multirow{1}{*}{Airplane}
&	26.88 	&	27.08 	&	26.78 	&	26.70 	&	27.12 	&	26.82 	&	26.49 	&	\textbf{27.23} 	&	25.76 	&	25.96 	&	25.63 	&	25.64 	&	25.98 	&	25.67 	&	25.16 	&	\textbf{26.08}

\\
\hline
				\multirow{1}{*}{Barbara}
&	27.26 	&	25.99 	&	27.25 	&	27.59 	&	27.43 	&	27.42 	&	26.97 	&	\textbf{27.79} 	&	26.42 	&	24.86 	&	26.13 	&	26.42 	&	26.27 	&	26.17 	&	25.66 	&	\textbf{26.66}
\\
\hline
				\multirow{1}{*}{boats}
&	27.76 	&	27.42 	&	27.52 	&	27.55 	&	27.90 	&	27.69 	&	27.09 	&	\textbf{27.91} 	&	26.74 	&	26.31 	&	26.37 	&	26.38 	&	26.82 	&	26.41 	&	25.81 	&	\textbf{26.93}

\\
\hline
				\multirow{1}{*}{C. man}
&	27.18 	&	27.03 	&	27.10 	&	26.56 	&	\textbf{27.34} 	&	26.93 	&	26.75 	&	27.33 	&	26.13 	&	26.01 	&	26.13 	&	25.62 	&	\textbf{26.46} 	&	25.93 	&	25.35 	&	26.36
\\
\hline
				\multirow{1}{*}{Fence}
&	26.84 	&	25.74 	&	26.76 	&	26.42 	&	26.91 	&	26.72 	&	26.47 	&	\textbf{27.16} 	&	25.92 	&	24.57 	&	25.77 	&	25.49 	&	25.94 	&	25.52 	&	25.22 	&	\textbf{26.24}
\\
\hline
				\multirow{1}{*}{Foreman}
&	31.29 	&	30.28 	&	31.52 	&	30.90 	&	31.55 	&	31.64 	&	30.04 	&	\textbf{31.99} 	&	30.36 	&	29.20 	&	30.41 	&	29.60 	&	30.45 	&	30.00 	&	28.69 	&	\textbf{30.75}
\\
\hline
				\multirow{1}{*}{House}
&	30.65 	&	29.89 	&	30.79 	&	30.25 	&	31.02 	&	30.68 	&	29.49 	&	\textbf{31.38} 	&	29.69 	&	28.79 	&	29.61 	&	28.99 	&	29.93 	&	29.17 	&	28.00 	&	\textbf{30.39}
\\
\hline
				\multirow{1}{*}{J. Bean}
&	30.21 	&	29.96 	&	30.49 	&	29.97 	&	30.39 	&	30.45 	&	29.23 	&	\textbf{30.91} 	&	29.26 	&	28.73 	&	29.24 	&	28.66 	&	29.20 	&	28.94 	&	27.77 	&	\textbf{29.55}
\\
\hline
				\multirow{1}{*}{Leaves}
&	25.69 	&	25.62 	&	26.20 	&	25.45 	&	26.29 	&	26.06 	&	25.91 	&	\textbf{26.90} 	&	24.68 	&	24.39 	&	24.94 	&	24.28 	&	25.03 	&	24.63 	&	24.22 	&	\textbf{25.56}
\\
\hline
				\multirow{1}{*}{Lena}
&	27.82 	&	27.78 	&	28.00 	&	27.78 	&	28.22 	&	28.04 	&	27.36 	&	\textbf{28.11} 	&	26.90 	&	26.68 	&	26.94 	&	26.70 	&	27.15 	&	26.78 	&	26.15 	&	\textbf{27.10}
\\
\hline
				\multirow{1}{*}{Lin}
&	29.52 	&	29.32 	&	29.27 	&	29.40 	&	\textbf{29.73} 	&	29.55 	&	28.63 	&	29.53 	&	28.71 	&	28.26 	&	28.23 	&	28.31 	&	\textbf{28.79} 	&	28.35 	&	27.33 	&	28.51
\\
\hline
				\multirow{1}{*}{Monarch}
&	26.72 	&	26.89 	&	26.81 	&	26.43 	&	27.02 	&	27.00 	&	26.68 	&	\textbf{27.34} 	&	25.82 	&	25.78 	&	25.73 	&	25.41 	&	26.00 	&	25.78 	&	25.30 	&	\textbf{26.22}
\\
\hline
				\multirow{1}{*}{Parrot}
&	28.64 	&	28.60 	&	28.77 	&	28.38 	&	28.95 	&	28.93 	&	28.30 	&	\textbf{29.09} 	&	27.88 	&	27.53 	&	27.67 	&	27.26 	&	27.91 	&	27.67 	&	26.77 	&	\textbf{28.06}

\\
\hline
				\multirow{1}{*}{Starfish}
&	26.06 	&	26.12 	&	26.17 	&	25.70 	&	26.21 	&	26.00 	&	25.87 	&	\textbf{26.49} 	&	25.04 	&	25.05 	&	25.06 	&	24.71 	&	25.11 	&	24.84 	&	24.58 	&	\textbf{25.39}

\\
\hline
				\multirow{1}{*}{\textbf{Average}}
&	28.04 	&	27.69 	&	28.10 	&	27.79 	&	28.29 	&	28.14 	&	27.52 	&	\textbf{28.51} 	&	27.09 	&	26.58 	&	26.99 	&	26.67 	&	27.22 	&	26.85 	&	26.14 	&	\textbf{27.41}
\\
\hline

				\multicolumn{1}{|c|}{}&\multicolumn{8}{|c||}{$\sigma_n=75$}&\multicolumn{8}{|c|}{$\sigma_n=100$}\\
				\hline
				\multirow{2}{*}{\textbf{{Images}}}&\multirow{2}{*}{\textbf{{BM3D}}}
				&\multirow{2}{*}{\textbf{{EPLL}}}&\multirow{2}{*}{\textbf{{NCSR}}}&\multirow{2}{*}{\textbf{{Plow}}}&\multirow{2}{*}{\textbf{{PGPD}}}
				&\multirow{2}{*}{\textbf{{OGLR}}}&{\textbf{{GSR-}}}&{\textbf{{GSR-}}}&\multirow{2}{*}{\textbf{{BM3D}}}&\multirow{2}{*}{\textbf{{EPLL}}}
                &\multirow{2}{*}{\textbf{{NCSR}}}&\multirow{2}{*}{\textbf{{Plow}}}&\multirow{2}{*}{\textbf{{PGPD}}}&\multirow{2}{*}{\textbf{{OGLR}}}
                &{\textbf{{GSR-}}}&{\textbf{{GSR-}}}\\
				& &  & & & & &{\textbf{NNM}} &{\textbf{WNNM}} & & & & & & & {\textbf{NNM}}& {\textbf{WNNM}}  \\
				\hline
				\multirow{1}{*}{Airplane}
&	23.99 	&	22.78 	&	23.76 	&	23.67 	&	24.15 	&	23.79 	&	23.15 	&	\textbf{24.16} 	&	22.89 	&	22.78 	&	22.60 	&	22.30 	&	23.02 	&	22.31 	&	19.09 	&	\textbf{23.06}
\\
\hline
				\multirow{1}{*}{Barbara}
&	24.53 	&	21.89 	&	24.06 	&	24.30 	&	24.39 	&	\textbf{24.52} 	&	23.58 	&	24.44 	&	23.20 	&	21.89 	&	22.70 	&	22.86 	&	23.11 	&	22.73 	&	22.01 	&	\textbf{23.25}
\\
\hline
				\multirow{1}{*}{boats}
&	24.82 	&	23.01 	&	24.44 	&	24.23 	&	24.83 	&	24.40 	&	23.64 	&	\textbf{24.92} 	&	23.47 	&	23.01 	&	22.98 	&	22.69 	&	23.47 	&	22.74 	&	22.07 	&	\textbf{23.50}
\\
\hline
				\multirow{1}{*}{C. man}
&	24.33 	&	22.84 	&	24.20 	&	23.64 	&	\textbf{24.64} 	&	24.00 	&	23.30 	&	24.57 	&	23.08 	&	22.84 	&	22.91 	&	22.22 	&	23.23 	&	22.50 	&	21.56 	&	\textbf{23.31}
\\
\hline
				\multirow{1}{*}{Fence}
&	24.22 	&	21.10 	&	23.75 	&	23.57 	&	24.18 	&	23.94 	&	23.22 	&	\textbf{24.46} 	&	22.92 	&	21.10 	&	22.23 	&	22.17 	&	22.87 	&	22.36 	&	21.62 	&	\textbf{23.29}
\\
\hline
				\multirow{1}{*}{Foreman}
&	28.07 	&	25.91 	&	28.18 	&	27.15 	&	28.39 	&	27.96 	&	26.18 	&	\textbf{28.77} 	&	26.51 	&	25.91 	&	26.55 	&	25.55 	&	26.81 	&	26.11 	&	24.79 	&	\textbf{27.44}
\\
\hline
				\multirow{1}{*}{House}
&	27.51 	&	25.21 	&	27.16 	&	26.52 	&	27.81 	&	27.10 	&	25.56 	&	\textbf{28.41} 	&	25.87 	&	25.21 	&	25.49 	&	24.72 	&	26.17 	&	25.07 	&	23.66 	&	\textbf{26.71}
\\
\hline
				\multirow{1}{*}{J. Bean}
&	27.22 	&	25.16 	&	27.15 	&	26.23 	&	27.07 	&	26.48 	&	25.23 	&	\textbf{27.47} 	&	25.80 	&	25.16 	&	25.61 	&	24.55 	&	25.66 	&	24.57 	&	23.73 	&	\textbf{26.20}
\\
\hline
				\multirow{1}{*}{Leaves}
&	22.49 	&	20.26 	&	22.60 	&	22.02 	&	22.61 	&	22.20 	&	21.79 	&	\textbf{23.24} 	&	20.90 	&	20.26 	&	20.84 	&	20.43 	&	20.95 	&	20.28 	&	19.57 	&	\textbf{21.56}
\\
\hline
				\multirow{1}{*}{Lena}
&	25.17 	&	23.46 	&	25.02 	&	24.64 	&	25.30 	&	24.90 	&	24.08 	&	\textbf{25.35} 	&	23.87 	&	23.46 	&	23.63 	&	23.19 	&	24.02 	&	23.18 	&	22.30 	&	\textbf{24.29}
\\
\hline
				\multirow{1}{*}{Lin}
&	26.96 	&	25.05 	&	26.22 	&	26.08 	&	\textbf{27.05} 	&	26.36 	&	25.07 	&	26.88 	&	25.60 	&	25.05 	&	24.85 	&	24.47 	&	\textbf{25.66} 	&	24.63 	&	23.36 	&	25.60
\\
\hline
				\multirow{1}{*}{Monarch}
&	23.91 	&	22.24 	&	23.67 	&	23.34 	&	24.00 	&	23.73 	&	23.06 	&	\textbf{24.28} 	&	22.52 	&	22.24 	&	22.10 	&	21.83 	&	22.56 	&	21.87 	&	21.03 	&	\textbf{22.86}
\\
\hline
				\multirow{1}{*}{Parrot}
&	25.94 	&	24.08 	&	25.45 	&	25.15 	&	25.98 	&	25.74 	&	24.54 	&	\textbf{26.08} 	&	24.60 	&	24.08 	&	23.94 	&	23.65 	&	24.52 	&	24.03 	&	22.84 	&	\textbf{24.67}
\\
\hline
				\multirow{1}{*}{Starfish}
&	23.27 	&	21.92 	&	23.18 	&	22.82 	&	23.23 	&	23.00 	&	22.52 	&	\textbf{23.32} 	&	\textbf{22.10} 	&	21.92 	&	21.91 	&	21.48 	&	22.08 	&	21.52 	&	20.97 	&	22.02
\\
\hline
				\multirow{1}{*}{\textbf{Average}}
&	25.17 	&	23.21 	&	24.92 	&	24.53 	&	25.26 	&	24.87 	&	23.92 	&	\textbf{25.45} 	&	23.81 	&	23.21 	&	23.45 	&	23.01 	&	23.87 	&	23.14 	&	22.04 	&	\textbf{24.13}
\\
\hline
			\end{tabular}}
			\label{Tab:1}
\vspace{-3mm}
		\end{table*}

	\begin{table*}[!htbp]
\vspace{-3mm}
			\caption{SSIM comparison of BM3D \cite{39}, EPLL \cite{49}, NCSR \cite{50}, Plow \cite{51}, PGPD \cite{52}, OGLR \cite{53}, GSR-NNM and GSR-WNNM for image denoising.}
		\centering
		\resizebox{1.00\textwidth}{!}
		{
			\begin{tabular}{|c|c|c|c|c|c|c|c|c||c|c|c|c|c|c|c|c|c|c|c|c|c|c|c|c|}
				\hline
				\multicolumn{1}{|c|}{}&\multicolumn{8}{|c||}{$\sigma_n=20$}&\multicolumn{8}{|c|}{$\sigma_n=30$}\\
				\hline
				\multirow{2}{*}{\textbf{{Images}}}&\multirow{2}{*}{\textbf{{BM3D}}}
				&\multirow{2}{*}{\textbf{{EPLL}}}&\multirow{2}{*}{\textbf{{NCSR}}}&\multirow{2}{*}{\textbf{{Plow}}}&\multirow{2}{*}{\textbf{{PGPD}}}
				&\multirow{2}{*}{\textbf{{OGLR}}}&{\textbf{{GSR-}}}&{\textbf{{GSR-}}}&\multirow{2}{*}{\textbf{{BM3D}}}&\multirow{2}{*}{\textbf{{EPLL}}}
                &\multirow{2}{*}{\textbf{{NCSR}}}&\multirow{2}{*}{\textbf{{Plow}}}&\multirow{2}{*}{\textbf{{PGPD}}}&\multirow{2}{*}{\textbf{{OGLR}}}
                &{\textbf{{GSR-}}}&{\textbf{{GSR-}}}\\
				& &  & & & & &{\textbf{NNM}} &{\textbf{WNNM}} & & & & & & & {\textbf{NNM}}& {\textbf{WNNM}}  \\
				\hline
				\multirow{1}{*}{Airplane}
&	0.9006 	&	0.9017 	&	0.9016 	&	0.8928 	&	0.8992 	&	0.8964 	&	0.8486 	&	\textbf{0.9042} 	&	0.8631 	&	0.8628 	&	0.8660 	&	0.8532 	&	0.8646 	&	0.8588 	&	0.7441 	&	\textbf{0.8713}
\\
\hline
				\multirow{1}{*}{Barbara}
&	0.9099 	&	0.8864 	&	0.9073 	&	0.9002 	&	0.9051 	&	0.9036 	&	0.8770 	&	\textbf{0.9122} 	&	0.8618 	&	0.8209 	&	0.8524 	&	0.8597 	&	0.8565 	&	0.8573 	&	0.7924 	&	\textbf{0.8674}
\\
\hline
				\multirow{1}{*}{boats}
&	\textbf{0.8890} 	&	0.8805 	&	0.8831 	&	0.8766 	&	0.8852 	&	0.8857 	&	0.8502 	&	0.8889 	&	\textbf{0.8424} 	&	0.8317 	&	0.8346 	&	0.8289 	&	0.8404 	&	0.8357 	&	0.7571 	&	0.8402
\\
\hline
				\multirow{1}{*}{C. man}
&	0.8755 	&	\textbf{0.8817} 	&	0.8765 	&	0.8614 	&	0.8624 	&	0.8777 	&	0.8118 	&	0.8612 	&	0.8373 	&	\textbf{0.8316} 	&	0.8382 	&	0.8216 	&	0.8259 	&	0.8313 	&	0.7114 	&	0.8276
\\
\hline
				\multirow{1}{*}{Fence}
&	0.8762 	&	0.8698 	&	0.8767 	&	0.8561 	&	0.8714 	&	\textbf{0.8807} 	&	0.8378 	&	0.8701 	&	\textbf{0.8326} 	&	0.8150 	&	0.8298 	&	0.8182 	&	0.8255 	&	0.8344 	&	0.7785 	&	0.8241

\\
\hline
				\multirow{1}{*}{Foreman}
&	0.9076 	&	0.8955 	&	0.9065 	&	0.9023 	&	0.9023 	&	0.9048 	&	0.8664 	&	\textbf{0.9098} 	&	0.8823 	&	0.8617 	&	0.8846 	&	0.8698 	&	0.8818 	&	0.8789 	&	0.7216 	&	\textbf{0.8950}
\\
\hline
				\multirow{1}{*}{House}
&	0.8726 	&	0.8609 	&	0.8735 	&	0.8710 	&	0.8693 	&	\textbf{0.8775} 	&	0.8325 	&	0.8688 	&	0.8480 	&	0.8338 	&	0.8479 	&	0.8383 	&	0.8471 	&	0.8448 	&	0.7118 	&	\textbf{0.8534}

\\
\hline
				\multirow{1}{*}{J. Bean}
&	0.9582 	&	0.9523 	&	\textbf{0.9632} 	&	0.9554 	&	0.9508 	&	0.9592 	&	0.8904 	&	0.9623 	&	0.9357 	&	0.9240 	&	0.9435 	&	0.9204 	&	0.9317 	&	0.9361 	&	0.7572 	&	\textbf{0.9509}

\\
\hline
				\multirow{1}{*}{Leaves}
&	0.9534 	&	0.9480 	&	0.9555 	&	0.9376 	&	0.9562 	&	0.9521 	&	0.9360 	&	\textbf{0.9618} 	&	0.9278 	&	0.9197 	&	0.9311 	&	0.9057 	&	0.9300 	&	0.9266 	&	0.8780 	&	\textbf{0.9383}

\\
\hline
				\multirow{1}{*}{Lena}
&	0.8985 	&	0.8913 	&	0.8979 	&	0.8891 	&	0.8981 	&	0.8944 	&	0.8597 	&	\textbf{0.9006} 	&	0.8584 	&	0.8477 	&	0.8580 	&	0.8493 	&	0.8622 	&	0.8560 	&	0.7543 	&	\textbf{0.8647}

\\
\hline
				\multirow{1}{*}{Lin}
&	\textbf{0.9017} 	&	0.8942 	&	0.8983 	&	0.8982 	&	0.8910 	&	0.8990 	&	0.8404 	&	0.8927 	&	\textbf{0.8672} 	&	0.8546 	&	0.8632 	&	0.8588 	&	0.8606 	&	0.8592 	&	0.7055 	&	0.8627

\\
\hline
				\multirow{1}{*}{Monarch}
&	0.9179 	&	0.9166 	&	0.9192 	&	0.9097 	&	0.9187 	&	0.9171 	&	0.8921 	&	\textbf{0.9232} 	&	0.8822 	&	0.8789 	&	0.8829 	&	0.8714 	&	0.8853 	&	0.8831 	&	0.7980 	&	\textbf{0.8918}

\\
\hline
				\multirow{1}{*}{Parrot}
&	\textbf{0.9002} 	&	0.8924 	&	0.8995 	&	0.8952 	&	0.8945 	&	0.8941 	&	0.8568 	&	0.8981 	&	0.8705 	&	0.8569 	&	0.8705 	&	0.8617 	&	0.8681 	&	0.8609 	&	0.7337 	&	\textbf{0.8733}

\\
\hline
				\multirow{1}{*}{Starfish}
&	0.8748 	&	\textbf{0.8756} 	&	0.8748 	&	0.8561 	&	0.8756 	&	0.8676 	&	0.8509 	&	0.8752 	&	\textbf{0.8289} 	&	0.8248 	&	0.8283 	&	0.8075 	&	0.8277 	&	0.8195 	&	0.7725 	&	0.8289

\\
\hline
				\multirow{1}{*}{\textbf{Average}}
&	\textbf{0.9026} 	&	0.8962 	&	0.9024 	&	0.8930 	&	0.8985 	&	0.9007 	&	0.8608 	&	0.9021 	&	0.8670 	&	0.8546 	&	0.8665 	&	0.8546 	&	0.8648 	&	0.8631 	&	0.7583 	&	\textbf{0.8707}

\\
\hline
				\multicolumn{1}{|c|}{}&\multicolumn{8}{|c||}{$\sigma_n=40$}&\multicolumn{8}{|c|}{$\sigma_n=50$}\\
				\hline
				\multirow{2}{*}{\textbf{{Images}}}&\multirow{2}{*}{\textbf{{BM3D}}}
				&\multirow{2}{*}{\textbf{{EPLL}}}&\multirow{2}{*}{\textbf{{NCSR}}}&\multirow{2}{*}{\textbf{{Plow}}}&\multirow{2}{*}{\textbf{{PGPD}}}
				&\multirow{2}{*}{\textbf{{OGLR}}}&{\textbf{{GSR-}}}&{\textbf{{GSR-}}}&\multirow{2}{*}{\textbf{{BM3D}}}&\multirow{2}{*}{\textbf{{EPLL}}}
                &\multirow{2}{*}{\textbf{{NCSR}}}&\multirow{2}{*}{\textbf{{Plow}}}&\multirow{2}{*}{\textbf{{PGPD}}}&\multirow{2}{*}{\textbf{{OGLR}}}
                &{\textbf{{GSR-}}}&{\textbf{{GSR-}}}\\
				& &  & & & & &{\textbf{NNM}} &{\textbf{WNNM}} & & & & & & & {\textbf{NNM}}& {\textbf{WNNM}}  \\
				\hline
				\multirow{1}{*}{Airplane}
&	0.8277 	&	0.8264 	&	0.8330 	&	0.8122 	&	0.8345 	&	0.8289 	&	0.7439 	&	\textbf{0.8434} 	&	0.8044 	&	0.7922 	&	0.8066 	&	0.7698 	&	0.8059 	&	0.7848 	&	0.6839 	&	\textbf{0.8157}

\\
\hline
				\multirow{1}{*}{Barbara}
&	0.8070 	&	0.7533 	&	0.8006 	&	0.8141 	&	0.8077 	&	0.8172 	&	0.7639 	&	\textbf{0.8210} 	&	0.7698 	&	0.6943 	&	0.7572 	&	0.7663 	&	0.7613 	&	0.7630 	&	0.7004 	&	\textbf{0.7814}

\\
\hline
				\multirow{1}{*}{boats}
&	0.7997 	&	0.7888 	&	0.7906 	&	0.7832 	&	\textbf{0.8021} 	&	0.7971 	&	0.7412 	&	0.8014 	&	0.7667 	&	0.7504 	&	0.7541 	&	0.7396 	&	0.7683 	&	0.7477 	&	0.6830 	&	\textbf{0.7747}

\\
\hline
				\multirow{1}{*}{C. man}
&	\textbf{0.8057} 	&	0.7932 	&	0.8019 	&	0.7832 	&	0.7997 	&	0.7953 	&	0.7171 	&	0.8040 	&	0.7828 	&	0.7617 	&	\textbf{0.7832} 	&	0.7459 	&	0.7774 	&	0.7561 	&	0.6519 	&	0.7791

\\
\hline
				\multirow{1}{*}{Fence}
&	0.7961 	&	0.7640 	&	0.7805 	&	0.7828 	&	0.7908 	&	\textbf{0.7975} 	&	0.7536 	&	0.7918 	&	0.7621 	&	0.7162 	&	0.7476 	&	0.7496 	&	0.7573 	&	0.7565 	&	0.6988 	&	\textbf{0.7633}

\\
\hline
				\multirow{1}{*}{Foreman}
&	0.8565 	&	0.8315 	&	0.8723 	&	0.8354 	&	0.8621 	&	0.8610 	&	0.7532 	&	\textbf{0.8775} 	&	0.8445 	&	0.8051 	&	0.8559 	&	0.7976 	&	0.8410 	&	0.8198 	&	0.6983 	&	\textbf{0.8578}

\\
\hline
				\multirow{1}{*}{House}
&	0.8256 	&	0.8089 	&	0.8323 	&	0.8058 	&	0.8302 	&	0.8218 	&	0.7342 	&	\textbf{0.8417} 	&	0.8122 	&	0.7845 	&	0.8160 	&	0.7699 	&	0.8125 	&	0.7824 	&	0.6780 	&	\textbf{0.8283}

\\
\hline
				\multirow{1}{*}{J. Bean}
&	0.9122 	&	0.8956 	&	0.9296 	&	0.8847 	&	0.9133 	&	0.9137 	&	0.7916 	&	\textbf{0.9341} 	&	0.9006 	&	0.8677 	&	\textbf{0.9134} 	&	0.8430 	&	0.8934 	&	0.8737 	&	0.7293 	&	0.9119

\\
\hline
				\multirow{1}{*}{Leaves}
&	0.8961 	&	0.8916 	&	0.9028 	&	0.8701 	&	0.9039 	&	0.8902 	&	0.8694 	&	\textbf{0.9177} 	&	0.8680 	&	0.8638 	&	0.8787 	&	0.8354 	&	0.8794 	&	0.8484 	&	0.8250 	&	\textbf{0.8952}

\\
\hline
				\multirow{1}{*}{Lena}
&	0.8178 	&	0.8092 	&	0.8280 	&	0.8081 	&	0.8297 	&	0.8250 	&	0.7505 	&	\textbf{0.8303} 	&	0.7920 	&	0.7732 	&	0.8009 	&	0.7691 	&	0.7990 	&	0.7764 	&	0.6966 	&	\textbf{0.8028}

\\
\hline
				\multirow{1}{*}{Lin}
&	\textbf{0.8369} 	&	0.8210 	&	0.8385 	&	0.8197 	&	0.8351 	&	0.8301 	&	0.7263 	&	0.8348 	&	0.8170 	&	0.7908 	&	0.8171 	&	0.7806 	&	\textbf{0.8118} 	&	0.7871 	&	0.6649 	&	0.8110

\\
\hline
				\multirow{1}{*}{Monarch}
&	0.8446 	&	0.8441 	&	0.8522 	&	0.8316 	&	0.8549 	&	0.8512 	&	0.7944 	&	\textbf{0.8590} 	&	0.8200 	&	0.8124 	&	0.8252 	&	0.7910 	&	0.8269 	&	0.8038 	&	0.7428 	&	\textbf{0.8286}

\\
\hline
				\multirow{1}{*}{Parrot}
&	0.8428 	&	0.8265 	&	0.8491 	&	0.8251 	&	0.8464 	&	0.8369 	&	0.7532 	&	\textbf{0.8511} 	&	0.8273 	&	0.7998 	&	0.8310 	&	0.7872 	&	0.8246 	&	0.7949 	&	0.6952 	&	\textbf{0.8328}

\\
\hline
				\multirow{1}{*}{Starfish}
&	0.7828 	&	0.7802 	&	0.7812 	&	0.7608 	&	0.7855 	&	0.7773 	&	0.7448 	&	\textbf{0.7914} 	&	0.7433 	&	0.7392 	&	0.7440 	&	0.7175 	&	0.7457 	&	0.7258 	&	0.6887 	&	\textbf{0.7567}

\\
\hline
				\multirow{1}{*}{\textbf{Average}}
&	0.8323 	&	0.8167 	&	0.8352 	&	0.8155 	&	0.8354 	&	0.8317 	&	0.7598 	&	\textbf{0.8428} 	&	0.8079 	&	0.7822 	&	0.8094 	&	0.7759 	&	0.8075 	&	0.7872 	&	0.7026 	&	\textbf{0.8171}

\\
\hline

				\multicolumn{1}{|c|}{}&\multicolumn{8}{|c||}{$\sigma_n=75$}&\multicolumn{8}{|c|}{$\sigma_n=100$}\\
				\hline
				\multirow{2}{*}{\textbf{{Images}}}&\multirow{2}{*}{\textbf{{BM3D}}}
				&\multirow{2}{*}{\textbf{{EPLL}}}&\multirow{2}{*}{\textbf{{NCSR}}}&\multirow{2}{*}{\textbf{{Plow}}}&\multirow{2}{*}{\textbf{{PGPD}}}
				&\multirow{2}{*}{\textbf{{OGLR}}}&{\textbf{{GSR-}}}&{\textbf{{GSR-}}}&\multirow{2}{*}{\textbf{{BM3D}}}&\multirow{2}{*}{\textbf{{EPLL}}}
                &\multirow{2}{*}{\textbf{{NCSR}}}&\multirow{2}{*}{\textbf{{Plow}}}&\multirow{2}{*}{\textbf{{PGPD}}}&\multirow{2}{*}{\textbf{{OGLR}}}
                &{\textbf{{GSR-}}}&{\textbf{{GSR-}}}\\
				& &  & & & & &{\textbf{NNM}} &{\textbf{WNNM}} & & & & & & & {\textbf{NNM}}& {\textbf{WNNM}}  \\
				\hline
				\multirow{1}{*}{Airplane}
&	0.7488 	&	0.6523 	&	0.7547 	&	0.6589 	&	0.7492 	&	0.7174 	&	0.5493 	&	\textbf{0.7698} 	&	0.7036 	&	0.6523 	&	0.7107 	&	0.5698 	&	0.6947 	&	0.6400 	&	0.5005 	&	\textbf{0.7302}

\\
\hline
				\multirow{1}{*}{Barbara}
&	0.6798 	&	0.5135 	&	0.6616 	&	0.6548 	&	0.6729 	&	\textbf{0.6791} 	&	0.5691 	&	0.6739 	&	0.6092 	&	0.5135 	&	0.5960 	&	0.5647 	&	0.6039 	&	0.5755 	&	0.5026 	&	\textbf{0.6171}

\\
\hline
				\multirow{1}{*}{boats}
&	0.6939 	&	0.5988 	&	0.6876 	&	0.6386 	&	0.6963 	&	0.6637 	&	0.5524 	&	\textbf{0.7058} 	&	0.6375 	&	0.5988 	&	0.6294 	&	0.5548 	&	0.6355 	&	0.5764 	&	0.4880 	&	\textbf{0.6510}

\\
\hline
				\multirow{1}{*}{C. man}
&	0.7340 	&	0.6351 	&	0.7412 	&	0.6311 	&	0.7301 	&	0.6792 	&	0.5164 	&	\textbf{0.7467} 	&	0.6928 	&	0.6351 	&	0.7067 	&	0.5304 	&	0.6776 	&	0.6088 	&	0.4671 	&	\textbf{0.7158}

\\
\hline
				\multirow{1}{*}{Fence}
&	0.6962 	&	0.5252 	&	0.6742 	&	0.6586 	&	0.6872 	&	0.6848 	&	0.5890 	&	\textbf{0.6984} 	&	0.6362 	&	0.5252 	&	0.6009 	&	0.5727 	&	0.6226 	&	0.6119 	&	0.5044 	&	\textbf{0.6501}

\\
\hline
				\multirow{1}{*}{Foreman}
&	0.7933 	&	0.6949 	&	0.8171 	&	0.7067 	&	0.7965 	&	0.7673 	&	0.5524 	&	\textbf{0.8247} 	&	0.7489 	&	0.6949 	&	0.7833 	&	0.6329 	&	0.7452 	&	0.6983 	&	0.5160 	&	\textbf{0.8007}

\\
\hline
				\multirow{1}{*}{House}
&	0.7645 	&	0.6695 	&	0.7749 	&	0.6733 	&	0.7709 	&	0.7230 	&	0.5439 	&	\textbf{0.8038} 	&	0.7203 	&	0.6695 	&	0.7397 	&	0.5874 	&	0.7195 	&	0.6373 	&	0.4918 	&	\textbf{0.7756}

\\
\hline
				\multirow{1}{*}{J. Bean}
&	0.8573 	&	0.7429 	&	0.8792 	&	0.7422 	&	0.8503 	&	0.8088 	&	0.5796 	&	\textbf{0.8857} 	&	0.8181 	&	0.7429 	&	0.8472 	&	0.6574 	&	0.7999 	&	0.7331 	&	0.5341 	&	\textbf{0.8640}

\\
\hline
				\multirow{1}{*}{Leaves}
&	0.8072 	&	0.7163 	&	0.8234 	&	0.7512 	&	0.8121 	&	0.7763 	&	0.7265 	&	\textbf{0.8473} 	&	0.7482 	&	0.7163 	&	0.7622 	&	0.6814 	&	0.7469 	&	0.6827 	&	0.6345 	&	\textbf{0.7964}

\\
\hline
				\multirow{1}{*}{Lena}
&	0.7288 	&	0.6345 	&	0.7415 	&	0.6723 	&	0.7356 	&	0.7061 	&	0.5647 	&	\textbf{0.7484} 	&	0.6739 	&	0.6345 	&	0.6906 	&	0.5895 	&	0.6780 	&	0.6215 	&	0.5093 	&	\textbf{0.7124}

\\
\hline
				\multirow{1}{*}{Lin}
&	0.7673 	&	0.6669 	&	0.7730 	&	0.6722 	&	0.7669 	&	0.7189 	&	0.5175 	&	\textbf{0.7821} 	&	0.7262 	&	0.6669 	&	0.7393 	&	0.5907 	&	0.7151 	&	0.6406 	&	0.4858 	&	\textbf{0.7547}

\\
\hline
				\multirow{1}{*}{Monarch}
&	0.7557 	&	0.6771 	&	0.7648 	&	0.6917 	&	0.7642 	&	0.7378 	&	0.6206 	&	\textbf{0.7754} 	&	0.7021 	&	0.6771 	&	0.7109 	&	0.6102 	&	0.7029 	&	0.6419 	&	0.5596 	&	\textbf{0.7307}

\\
\hline
				\multirow{1}{*}{Parrot}
&	0.7771 	&	0.6844 	&	0.7892 	&	0.6859 	&	0.7775 	&	0.7333 	&	0.5567 	&	\textbf{0.8005} 	&	0.7345 	&	0.6844 	&	0.7518 	&	0.6096 	&	0.7251 	&	0.6531 	&	0.5197 	&	\textbf{0.7713}

\\
\hline
				\multirow{1}{*}{Starfish}
&	0.6670 	&	0.5799 	&	0.6685 	&	0.6192 	&	0.6638 	&	0.6446 	&	0.5617 	&	\textbf{0.6717} 	&	0.6053 	&	0.5799 	&	0.6062 	&	0.5403 	&	0.6018 	&	0.5528 	&	0.4979 	&	\textbf{0.6090}

\\
\hline
				\multirow{1}{*}{\textbf{Average}}
&	0.7479 	&	0.6422 	&	0.7536 	&	0.6755 	&	0.7481 	&	0.7172 	&	0.5714 	&	\textbf{0.7667} 	&	0.6969 	&	0.6422 	&	0.7054 	&	0.5923 	&	0.6906 	&	0.6339 	&	0.5151 	&	\textbf{0.7271}

\\
\hline
         \end{tabular}}
         \vspace{-3mm}
			\label{Tab:2}
		\end{table*}

	\begin{table*}[!htbp]
\vspace{-3mm}
			\caption{PSNR ($\textnormal{d}$B) comparison of SALSA \cite{54}, BPFA \cite{44}, IPPO \cite{55}, JSM \cite{56}, Aloha \cite{57}, NGS \cite{58}, GSR-NNM and GSR-WNNM for image inpainting.}
		\centering
		\resizebox{1.00\textwidth}{!}
		{
			\begin{tabular}{|c|c|c|c|c|c|c|c|c||c|c|c|c|c|c|c|c|c|c|c|c|c|c|c|c|}
				\hline
				\multicolumn{1}{|c|}{}&\multicolumn{8}{|c||}{80\% pixels missing}&\multicolumn{8}{|c|}{70\% pixels missing}\\
				\hline
				\multirow{2}{*}{\textbf{{Images}}}&\multirow{2}{*}{\textbf{{SALSA}}}
				&\multirow{2}{*}{\textbf{{BPFA}}}&\multirow{2}{*}{\textbf{{IPPO}}}&\multirow{2}{*}{\textbf{{JSM}}}&\multirow{2}{*}{\textbf{{Aloha}}}
				&\multirow{2}{*}{\textbf{{NGS}}}&{\textbf{{GSR-}}}&{\textbf{{GSR-}}}&\multirow{2}{*}{\textbf{{SALSA}}}&\multirow{2}{*}{\textbf{{BPFA}}}
                &\multirow{2}{*}{\textbf{{IPPO}}}&\multirow{2}{*}{\textbf{{JSM}}}&\multirow{2}{*}{\textbf{{Aloha}}}&\multirow{2}{*}{\textbf{{NGS}}}
                &{\textbf{{GSR-}}}&{\textbf{{GSR-}}}\\
				& &  & & & & &{\textbf{NNM}} &{\textbf{WNNM}} & & & & & & & {\textbf{NNM}}& {\textbf{WNNM}}  \\
				\hline
				\multirow{1}{*}{Cowboy}
&	23.72 	&	24.93 	&	25.38 	&	25.40 	&	25.06 	&	24.21 	&	25.27 	&	\textbf{25.71} 	&	25.70 	&	26.76 	&	27.40 	&	27.11 	&	27.24 	&	26.19 	&	25.27 	&	\textbf{27.90}
\\
\hline
				\multirow{1}{*}{Light}
&	18.27 	&	19.23 	&	21.49 	&	20.23 	&	21.50 	&	18.52 	&	20.39 	&	\textbf{22.09} 	&	19.32 	&	21.58 	&	23.47 	&	23.12 	&	23.17 	&	20.78 	&	20.39 	&	\textbf{24.00}
\\
\hline
				\multirow{1}{*}{Mickey}
&	24.46 	&	24.53 	&	26.33 	&	26.09 	&	25.33 	&	24.50 	&	25.77 	&	\textbf{26.66} 	&	25.98 	&	26.16 	&	28.59 	&	28.25 	&	27.11 	&	26.68 	&	25.77 	&	\textbf{29.16}
\\
\hline
				\multirow{1}{*}{Butterfly}
&	22.85 	&	24.04 	&	25.13 	&	25.57 	&	24.88 	&	23.85 	&	25.53 	&	\textbf{26.46} 	&	25.06 	&	26.68 	&	27.68 	&	27.97 	&	27.29 	&	26.36 	&	25.53 	&	\textbf{29.19}
\\
\hline
				\multirow{1}{*}{Haight}
&	18.57 	&	19.42 	&	20.90 	&	21.37 	&	20.62 	&	18.76 	&	20.73 	&	\textbf{21.43} 	&	19.95 	&	21.46 	&	23.02 	&	23.01 	&	22.12 	&	21.03 	&	20.73 	&	\textbf{23.56}
\\
\hline
				\multirow{1}{*}{Lake}
&	24.94 	&	25.82 	&	25.48 	&	25.82 	&	25.32 	&	25.10 	&	25.73 	&	\textbf{26.15} 	&	26.76 	&	27.93 	&	27.56 	&	27.88 	&	27.58 	&	27.01 	&	25.73 	&	\textbf{28.55}
\\
\hline
				\multirow{1}{*}{Leaves}
&	22.03 	&	23.78 	&	25.56 	&	26.18 	&	25.90 	&	23.87 	&	25.57 	&	\textbf{27.10} 	&	24.36 	&	26.98 	&	28.58 	&	29.28 	&	29.04 	&	26.44 	&	25.57 	&	\textbf{30.55}
\\
\hline
				\multirow{1}{*}{Starfish}
&	25.70 	&	26.79 	&	26.30 	&	27.07 	&	26.33 	&	26.17 	&	26.98 	&	\textbf{27.66} 	&	27.55 	&	28.93 	&	28.91 	&	29.36 	&	28.22 	&	28.35 	&	26.98 	&	\textbf{30.07}
\\
\hline
				\multirow{1}{*}{Flower}
&	26.81 	&	27.30 	&	27.70 	&	27.41 	&	27.49 	&	27.03 	&	27.74 	&	\textbf{27.75} 	&	28.24 	&	28.92 	&	29.20 	&	29.06 	&	29.02 	&	28.48 	&	27.74 	&	\textbf{29.96}
\\
\hline
				\multirow{1}{*}{Nanna}
&	24.12 	&	24.71 	&	25.60 	&	25.33 	&	25.54 	&	24.58 	&	25.45 	&	\textbf{25.89} 	&	25.44 	&	26.62 	&	27.44 	&	27.34 	&	27.43 	&	26.35 	&	25.45 	&	\textbf{28.11}
\\
\hline
				\multirow{1}{*}{Corn}
&	24.28 	&	25.54 	&	25.14 	&	25.58 	&	25.60 	&	24.74 	&	25.67 	&	\textbf{26.94} 	&	26.11 	&	27.82 	&	27.77 	&	27.66 	&	27.95 	&	26.77 	&	25.67 	&	\textbf{28.84}
\\
\hline
				\multirow{1}{*}{Girl}
&	23.79 	&	24.80 	&	25.31 	&	25.18 	&	25.16 	&	24.27 	&	25.18 	&	\textbf{25.71} 	&	25.47 	&	26.86 	&	27.43 	&	27.20 	&	27.08 	&	26.18 	&	25.18 	&	\textbf{28.06}
\\
\hline
				\multirow{1}{*}{Fireman}
&	24.38 	&	24.88 	&	25.56 	&	25.31 	&	25.03 	&	24.54 	&	25.39 	&	\textbf{25.69} 	&	25.82 	&	26.55 	&	27.44 	&	27.16 	&	26.52 	&	26.29 	&	25.39 	&	\textbf{27.60}
\\
\hline
				\multirow{1}{*}{Mural}
&	23.15 	&	24.13 	&	25.66 	&	25.40 	&	25.23 	&	23.78 	&	25.03 	&	\textbf{26.05} 	&	25.00 	&	26.46 	&	27.92 	&	27.59 	&	27.33 	&	26.06 	&	25.03 	&	\textbf{28.39}
\\
\hline
				\multirow{1}{*}{\textbf{Average}}
&	23.36 	&	24.28 	&	25.11 	&	25.14 	&	24.93 	&	23.85 	&	25.03 	&	\textbf{25.81} 	&	25.05 	&	26.41 	&	27.31 	&	27.29 	&	26.93 	&	25.93 	&	25.03 	&	\textbf{28.14}
\\
\hline

				\multicolumn{1}{|c|}{}&\multicolumn{8}{|c||}{60\% pixels missing}&\multicolumn{8}{|c|}{50\% pixels missing}\\
				\hline
				\multirow{2}{*}{\textbf{{Images}}}&\multirow{2}{*}{\textbf{{SALSA}}}
				&\multirow{2}{*}{\textbf{{BPFA}}}&\multirow{2}{*}{\textbf{{IPPO}}}&\multirow{2}{*}{\textbf{{JSM}}}&\multirow{2}{*}{\textbf{{Aloha}}}
				&\multirow{2}{*}{\textbf{{NGS}}}&{\textbf{{GSR-}}}&{\textbf{{GSR-}}}&\multirow{2}{*}{\textbf{{SALSA}}}&\multirow{2}{*}{\textbf{{BPFA}}}
                &\multirow{2}{*}{\textbf{{IPPO}}}&\multirow{2}{*}{\textbf{{JSM}}}&\multirow{2}{*}{\textbf{{Aloha}}}&\multirow{2}{*}{\textbf{{NGS}}}
                &{\textbf{{GSR-}}}&{\textbf{{GSR-}}}\\
				& &  & & & & &{\textbf{NNM}} &{\textbf{WNNM}} & & & & & & & {\textbf{NNM}}& {\textbf{WNNM}}  \\
				\hline
				\multirow{1}{*}{Cowboy}
&	26.99 	&	28.42 	&	29.58 	&	28.89 	&	28.92 	&	27.78 	&	28.73 	&	\textbf{30.03} 	&	28.59 	&	30.21 	&	31.30 	&	30.75 	&	30.46 	&	29.32 	&	28.73 	&	\textbf{31.96}
\\
\hline
				\multirow{1}{*}{Light}
&	20.49 	&	23.65 	&	25.13 	&	24.83 	&	24.47 	&	22.78 	&	24.62 	&	\textbf{25.43} 	&	21.47 	&	25.71 	&	26.70 	&	26.48 	&	25.84 	&	24.62 	&	24.62 	&	\textbf{27.28}
\\
\hline
				\multirow{1}{*}{Mickey}
&	27.41 	&	27.83 	&	30.76 	&	29.85 	&	28.59 	&	28.09 	&	29.52 	&	\textbf{31.23} 	&	28.98 	&	29.43 	&	32.74 	&	31.96 	&	30.33 	&	29.75 	&	29.52 	&	\textbf{33.67}
\\
\hline
				\multirow{1}{*}{Butterfly}
&	26.79 	&	28.88 	&	29.85 	&	29.83 	&	29.16 	&	28.37 	&	29.80 	&	\textbf{31.27} 	&	28.52 	&	30.98 	&	31.69 	&	31.47 	&	30.78 	&	30.28 	&	29.80 	&	\textbf{33.00}
\\
\hline
				\multirow{1}{*}{Haight}
&	21.52 	&	23.33 	&	25.34 	&	24.70 	&	23.58 	&	22.81 	&	24.35 	&	\textbf{25.87} 	&	23.06 	&	25.40 	&	27.53 	&	26.67 	&	25.16 	&	24.50 	&	24.35 	&	\textbf{28.36}
\\
\hline
				\multirow{1}{*}{Lake}
&	28.14 	&	29.75 	&	29.30 	&	29.49 	&	29.24 	&	28.68 	&	29.25 	&	\textbf{30.33} 	&	29.69 	&	31.78 	&	30.98 	&	31.18 	&	31.17 	&	30.22 	&	29.25 	&	\textbf{32.26}
\\
\hline
				\multirow{1}{*}{Leaves}
&	26.29 	&	29.83 	&	30.88 	&	31.47 	&	31.41 	&	28.87 	&	30.55 	&	\textbf{32.89} 	&	28.11 	&	32.79 	&	33.32 	&	33.78 	&	34.01 	&	31.23 	&	30.55 	&	\textbf{35.41}
\\
\hline
				\multirow{1}{*}{Starfish}
&	29.09 	&	30.98 	&	31.09 	&	31.40 	&	30.19 	&	30.26 	&	30.78 	&	\textbf{32.28} 	&	30.90 	&	33.13 	&	33.10 	&	33.24 	&	31.85 	&	32.10 	&	30.78 	&	\textbf{34.27}
\\
\hline
				\multirow{1}{*}{Flower}
&	29.36 	&	30.61 	&	30.81 	&	30.52 	&	30.72 	&	29.84 	&	30.60 	&	\textbf{31.96} 	&	30.92 	&	32.55 	&	32.49 	&	32.04 	&	32.40 	&	31.40 	&	30.60 	&	\textbf{33.77}
\\
\hline
				\multirow{1}{*}{Nanna}
&	26.94 	&	28.63 	&	29.41 	&	29.09 	&	29.51 	&	28.06 	&	29.02 	&	\textbf{30.27} 	&	28.53 	&	30.68 	&	31.17 	&	30.75 	&	31.24 	&	29.71 	&	29.02 	&	\textbf{32.30}
\\
\hline
				\multirow{1}{*}{Corn}
&	27.75 	&	30.07 	&	29.75 	&	29.45 	&	29.83 	&	28.55 	&	29.33 	&	\textbf{31.52} 	&	29.39 	&	32.10 	&	31.76 	&	31.33 	&	31.89 	&	30.31 	&	29.33 	&	\textbf{33.69}
\\
\hline
				\multirow{1}{*}{Girl}
&	27.02 	&	28.75 	&	29.32 	&	29.01 	&	28.91 	&	27.83 	&	28.71 	&	\textbf{30.01} 	&	28.60 	&	30.58 	&	31.05 	&	30.68 	&	30.59 	&	29.60 	&	28.71 	&	\textbf{32.11}
\\
\hline
				\multirow{1}{*}{Fireman}
&	27.15 	&	28.23 	&	29.13 	&	28.79 	&	28.24 	&	27.67 	&	28.53 	&	\textbf{29.69} 	&	28.54 	&	30.12 	&	30.82 	&	30.37 	&	29.88 	&	29.22 	&	28.53 	&	\textbf{31.15}
\\
\hline
				\multirow{1}{*}{Mural}
&	26.66 	&	28.30 	&	29.57 	&	29.24 	&	28.92 	&	27.99 	&	28.85 	&	\textbf{30.09} 	&	28.20 	&	30.46 	&	31.11 	&	30.89 	&	30.28 	&	29.88 	&	28.85 	&	\textbf{31.79}
\\
\hline
				\multirow{1}{*}{\textbf{Average}}
&	26.54 	&	28.38 	&	29.28 	&	29.04 	&	28.69 	&	27.68 	&	28.76 	&	\textbf{30.20} 	&	28.11 	&	30.42 	&	31.13 	&	30.83 	&	30.42 	&	29.44 	&	28.76 	&	\textbf{32.22}
\\
\hline
			\end{tabular}}
			\label{Tab:3}
\vspace{-3mm}
		\end{table*}

	\begin{table*}[!htbp]
\vspace{-5mm}
			\caption{SSIM comparison of SALSA \cite{54}, BPFA \cite{44}, IPPO \cite{55}, JSM \cite{56}, Aloha \cite{57}, NGS \cite{58}, GSR-NNM and GSR-WNNM for image inpainting.}
		\centering
		\resizebox{1.00\textwidth}{!}
		{
			\begin{tabular}{|c|c|c|c|c|c|c|c|c||c|c|c|c|c|c|c|c|c|c|c|c|c|c|c|c|}
				\hline
				\multicolumn{1}{|c|}{}&\multicolumn{8}{|c||}{80\% pixels missing}&\multicolumn{8}{|c|}{70\% pixels missing}\\
				\hline
				\multirow{2}{*}{\textbf{{Images}}}&\multirow{2}{*}{\textbf{{SALSA}}}
				&\multirow{2}{*}{\textbf{{BPFA}}}&\multirow{2}{*}{\textbf{{IPPO}}}&\multirow{2}{*}{\textbf{{JSM}}}&\multirow{2}{*}{\textbf{{Aloha}}}
				&\multirow{2}{*}{\textbf{{NGS}}}&{\textbf{{GSR-}}}&{\textbf{{GSR-}}}&\multirow{2}{*}{\textbf{{SALSA}}}&\multirow{2}{*}{\textbf{{BPFA}}}
                &\multirow{2}{*}{\textbf{{IPPO}}}&\multirow{2}{*}{\textbf{{JSM}}}&\multirow{2}{*}{\textbf{{Aloha}}}&\multirow{2}{*}{\textbf{{NGS}}}
                &{\textbf{{GSR-}}}&{\textbf{{GSR-}}}\\
				& &  & & & & &{\textbf{NNM}} &{\textbf{WNNM}} & & & & & & & {\textbf{NNM}}& {\textbf{WNNM}}  \\
				\hline
				\multirow{1}{*}{Cowboy}
&	0.7875 	&	0.8459 	&	0.8642 	&	0.8615 	&	0.8580 	&	0.8392 	&	0.8589 	&	\textbf{0.8790} 	&	0.8742 	&	0.8950 	&	0.9118 	&	0.9075 	&	0.9056 	&	0.8903 	&	0.8589 	&	\textbf{0.9228}
\\
\hline
				\multirow{1}{*}{Light}
&	0.5617 	&	0.6281 	&	0.7827 	&	0.7254 	&	0.7734 	&	0.6041 	&	0.7174 	&	\textbf{0.8236} 	&	0.6665 	&	0.7858 	&	0.8612 	&	0.8528 	&	0.8496 	&	0.7538 	&	0.7174 	&	\textbf{0.8860}
\\
\hline
				\multirow{1}{*}{Mickey}
&	0.8145 	&	0.8117 	&	0.8678 	&	0.8598 	&	0.8300 	&	0.8230 	&	0.8541 	&	\textbf{0.8738} 	&	0.8621 	&	0.8661 	&	0.9151 	&	0.9064 	&	0.8797 	&	0.8791 	&	0.8541 	&	\textbf{0.9195}
\\
\hline
				\multirow{1}{*}{Butterfly}
&	0.8252 	&	0.8517 	&	0.8995 	&	0.9026 	&	0.8805 	&	0.8635 	&	0.9050 	&	\textbf{0.9184} 	&	0.8838 	&	0.9124 	&	0.9356 	&	0.9377 	&	0.9205 	&	0.9145 	&	0.9050 	&	\textbf{0.9473}
\\
\hline
				\multirow{1}{*}{Haight}
&	0.6961 	&	0.7307 	&	0.8251 	&	0.8320 	&	0.7955 	&	0.7351 	&	0.8077 	&	\textbf{0.8526} 	&	0.7773 	&	0.8269 	&	0.8878 	&	0.8831 	&	0.8557 	&	0.8303 	&	0.8077 	&	\textbf{0.9037}
\\
\hline
				\multirow{1}{*}{Lake}
&	0.8036 	&	0.8360 	&	0.8297 	&	0.8357 	&	0.8270 	&	0.8180 	&	0.8335 	&	\textbf{0.8538} 	&	0.8649 	&	0.8929 	&	0.8857 	&	0.8883 	&	0.8872 	&	0.8756 	&	0.8335 	&	\textbf{0.9062}
\\
\hline
				\multirow{1}{*}{Leaves}
&	0.7948 	&	0.8557 	&	0.9119 	&	0.9213 	&	0.9085 	&	0.8687 	&	0.9097 	&	\textbf{0.9319} 	&	0.8746 	&	0.9276 	&	0.9538 	&	0.9581 	&	0.9549 	&	0.9233 	&	0.9097 	&	\textbf{0.9641}
\\
\hline
				\multirow{1}{*}{Starfish}
&	0.8086 	&	0.8379 	&	0.8243 	&	0.8383 	&	0.8217 	&	0.8272 	&	0.8362 	&	\textbf{0.8435} 	&	0.8675 	&	0.8942 	&	0.8923 	&	0.8954 	&	0.8793 	&	0.8856 	&	0.8362 	&	\textbf{0.8997}
\\
\hline
				\multirow{1}{*}{Flower}
&	0.8024 	&	0.8193 	&	0.8477 	&	0.8204 	&	0.8322 	&	0.8140 	&	0.8390 	&	\textbf{0.8657} 	&	0.8526 	&	0.8746 	&	0.8974 	&	0.8801 	&	0.8871 	&	0.8647 	&	0.8390 	&	\textbf{0.9176}
\\
\hline
				\multirow{1}{*}{Nanna}
&	0.7819 	&	0.8011 	&	0.8311 	&	0.8214 	&	0.8314 	&	0.8015 	&	0.8212 	&	\textbf{0.8551} 	&	0.8369 	&	0.8681 	&	0.8876 	&	0.8827 	&	0.8907 	&	0.8624 	&	0.8212 	&	\textbf{0.9065}
\\
\hline
				\multirow{1}{*}{Corn}
&	0.7953 	&	0.8334 	&	0.8281 	&	0.8380 	&	0.8380 	&	0.8121 	&	0.8399 	&	\textbf{0.8790} 	&	0.8624 	&	0.8979 	&	0.9008 	&	0.8962 	&	0.9009 	&	0.8794 	&	0.8399 	&	\textbf{0.9141}
\\
\hline
				\multirow{1}{*}{Girl}
&	0.7535 	&	0.7846 	&	0.8140 	&	0.8031 	&	0.8055 	&	0.7757 	&	0.8016 	&	\textbf{0.8435} 	&	0.8250 	&	0.8588 	&	0.8823 	&	0.8739 	&	0.8728 	&	0.8483 	&	0.8016 	&	\textbf{0.9035}
\\
\hline
				\multirow{1}{*}{Fireman}
&	0.7254 	&	0.7430 	&	0.7875 	&	0.7667 	&	0.7663 	&	0.7483 	&	0.7711 	&	\textbf{0.8095} 	&	0.7965 	&	0.8200 	&	0.8554 	&	0.8393 	&	0.8355 	&	0.8196 	&	0.7711 	&	\textbf{0.8699}
\\
\hline
				\multirow{1}{*}{Mural}
&	0.7155 	&	0.7402 	&	0.7885 	&	0.7833 	&	0.7842 	&	0.7535 	&	0.7706 	&	\textbf{0.8132} 	&	0.7917 	&	0.8211 	&	0.8581 	&	0.8519 	&	0.8512 	&	0.8291 	&	0.7706 	&	\textbf{0.8731}
\\
\hline
				\multirow{1}{*}{\textbf{Average}}
&	0.7619 	&	0.7942 	&	0.8359 	&	0.8293 	&	0.8252 	&	0.7917 	&	0.8261 	&	\textbf{0.8602} 	&	0.8311 	&	0.8673 	&	0.8946 	&	0.8895 	&	0.8836 	&	0.8612 	&	0.8261 	&	\textbf{0.9096}
\\
\hline

				\multicolumn{1}{|c|}{}&\multicolumn{8}{|c||}{60\% pixels missing}&\multicolumn{8}{|c|}{50\% pixels missing}\\
				\hline
				\multirow{2}{*}{\textbf{{Images}}}&\multirow{2}{*}{\textbf{{SALSA}}}
				&\multirow{2}{*}{\textbf{{BPFA}}}&\multirow{2}{*}{\textbf{{IPPO}}}&\multirow{2}{*}{\textbf{{JSM}}}&\multirow{2}{*}{\textbf{{Aloha}}}
				&\multirow{2}{*}{\textbf{{NGS}}}&{\textbf{{GSR-}}}&{\textbf{{GSR-}}}&\multirow{2}{*}{\textbf{{SALSA}}}&\multirow{2}{*}{\textbf{{BPFA}}}
                &\multirow{2}{*}{\textbf{{IPPO}}}&\multirow{2}{*}{\textbf{{JSM}}}&\multirow{2}{*}{\textbf{{Aloha}}}&\multirow{2}{*}{\textbf{{NGS}}}
                &{\textbf{{GSR-}}}&{\textbf{{GSR-}}}\\
				& &  & & & & &{\textbf{NNM}} &{\textbf{WNNM}} & & & & & & & {\textbf{NNM}}& {\textbf{WNNM}}  \\
				\hline
				\multirow{1}{*}{Cowboy}
&	0.9064 	&	0.9281 	&	0.9438 	&	0.9368 	&	0.9362 	&	0.9246 	&	0.9345 	&	\textbf{0.9509} 	&	0.9344 	&	0.9505 	&	0.9611 	&	0.9577 	&	0.9547 	&	0.9464 	&	0.9345 	&	\textbf{0.9669}

\\
\hline
				\multirow{1}{*}{Light}
&	0.7529 	&	0.8664 	&	0.9057 	&	0.9010 	&	0.8910 	&	0.8452 	&	0.8926 	&	\textbf{0.9175} 	&	0.8069 	&	0.9164 	&	0.9350 	&	0.9322 	&	0.9212 	&	0.8983 	&	0.8926 	&	\textbf{0.9447}

\\
\hline
				\multirow{1}{*}{Mickey}
&	0.8977 	&	0.9033 	&	0.9425 	&	0.9327 	&	0.9127 	&	0.9119 	&	0.9303 	&	\textbf{0.9450} 	&	0.9243 	&	0.9312 	&	0.9606 	&	0.9537 	&	0.9371 	&	0.9386 	&	0.9303 	&	\textbf{0.9651}

\\
\hline
				\multirow{1}{*}{Butterfly}
&	0.9191 	&	0.9436 	&	0.9566 	&	0.9570 	&	0.9428 	&	0.9451 	&	0.9575 	&	\textbf{0.9630} 	&	0.9432 	&	0.9617 	&	0.9697 	&	0.9695 	&	0.9580 	&	0.9630 	&	0.9575 	&	\textbf{0.9765}

\\
\hline
				\multirow{1}{*}{Haight}
&	0.8392 	&	0.8844 	&	0.9287 	&	0.9195 	&	0.8968 	&	0.8842 	&	0.9141 	&	\textbf{0.9381} 	&	0.8880 	&	0.9226 	&	0.9540 	&	0.9459 	&	0.9244 	&	0.9190 	&	0.9141 	&	\textbf{0.9610}

\\
\hline
				\multirow{1}{*}{Lake}
&	0.9030 	&	0.9263 	&	0.9219 	&	0.9229 	&	0.9210 	&	0.9130 	&	0.9195 	&	\textbf{0.9394} 	&	0.9304 	&	0.9483 	&	0.9463 	&	0.9460 	&	0.9465 	&	0.9393 	&	0.9195 	&	\textbf{0.9589}

\\
\hline
				\multirow{1}{*}{Leaves}
&	0.9173 	&	0.9615 	&	0.9726 	&	0.9751 	&	0.9736 	&	0.9556 	&	0.9696 	&	\textbf{0.9780} 	&	0.9444 	&	0.9795 	&	0.9832 	&	0.9846 	&	0.9850 	&	0.9734 	&	0.9696 	&	\textbf{0.9868}

\\
\hline
				\multirow{1}{*}{Starfish}
&	0.9036 	&	0.9280 	&	0.9290 	&	0.9293 	&	0.9171 	&	0.9222 	&	0.9232 	&	\textbf{0.9315} 	&	0.9335 	&	0.9510 	&	0.9531 	&	0.9518 	&	0.9418 	&	0.9464 	&	0.9232 	&	\textbf{0.9535}

\\
\hline
				\multirow{1}{*}{Flower}
&	0.8898 	&	0.9144 	&	0.9295 	&	0.9177 	&	0.9230 	&	0.9053 	&	0.9211 	&	\textbf{0.9456} 	&	0.9187 	&	0.9439 	&	0.9521 	&	0.9441 	&	0.9461 	&	0.9349 	&	0.9211 	&	\textbf{0.9635}

\\
\hline
				\multirow{1}{*}{Nanna}
&	0.8823 	&	0.9132 	&	0.9245 	&	0.9184 	&	0.9273 	&	0.9046 	&	0.9163 	&	\textbf{0.9378} 	&	0.9173 	&	0.9432 	&	0.9485 	&	0.9440 	&	0.9497 	&	0.9338 	&	0.9163 	&	\textbf{0.9588}

\\
\hline
				\multirow{1}{*}{Corn}
&	0.9022 	&	0.9358 	&	0.9366 	&	0.9304 	&	0.9334 	&	0.9187 	&	0.9283 	&	\textbf{0.9556} 	&	0.9310 	&	0.9572 	&	0.9586 	&	0.9542 	&	0.9554 	&	0.9450 	&	0.9283 	&	\textbf{0.9715}

\\
\hline
				\multirow{1}{*}{Girl}
&	0.8754 	&	0.9053 	&	0.9226 	&	0.9156 	&	0.9150 	&	0.8943 	&	0.9099 	&	\textbf{0.9382} 	&	0.9108 	&	0.9346 	&	0.9477 	&	0.9433 	&	0.9420 	&	0.9281 	&	0.9099 	&	\textbf{0.9598}

\\
\hline
				\multirow{1}{*}{Fireman}
&	0.8507 	&	0.8717 	&	0.8976 	&	0.8876 	&	0.8834 	&	0.8690 	&	0.8840 	&	\textbf{0.9120} 	&	0.8891 	&	0.9109 	&	0.9287 	&	0.9206 	&	0.9184 	&	0.9082 	&	0.8840 	&	\textbf{0.9378}

\\
\hline
				\multirow{1}{*}{Mural}
&	0.8483 	&	0.8703 	&	0.8960 	&	0.8915 	&	0.8910 	&	0.8821 	&	0.8842 	&	\textbf{0.9075} 	&	0.8876 	&	0.9041 	&	0.9262 	&	0.9229 	&	0.9200 	&	0.9166 	&	0.8842 	&	\textbf{0.9353}

\\
\hline
				\multirow{1}{*}{\textbf{Average}}
&0.8777 	&	0.9109 	&	0.9291 	&	0.9240 	&	0.9189 	&	0.9054 	&	0.9204 	&	\textbf{0.9400} 	&	0.9114 	&	0.9396 	&	0.9518 	&	0.9479 	&	0.9429 	&	0.9351 	&	0.9204 	&	\textbf{0.9600}

\\
\hline
			\end{tabular}}
			\label{Tab:4}
\vspace{-3mm}
		\end{table*}

\subsection {Image Inpainting}

In this subsection, we show the experimental results of the proposed GSR-WNNM based image inpainting. We generate the mask with different partial random samples. We present the image inpainting results on four partial random samples, \ie, 80\%, 70\%, 60\% and 50\% pixels missing. The parameters are set as follows. The size of each patch $\sqrt{m} \times \sqrt{m}$ is set to be $7\times7$. The number of similar patches $k$ is set to 60. The search window for similar patches is set to $L$ = 20 and $\varepsilon = 10^{-16}$.  $\tau$ and $c$ are set to be (7.0e-5, 1.41), (5.8e-5, 1.10), (3.8e-5, 1.06) and(2.6e-5, 0.99) for 80\%, 70\%, 60\% and 50\% pixels missing, respectively.

\begin{figure}[!htbp]
\vspace{-3mm}
	\centerline{\includegraphics[width=9cm]{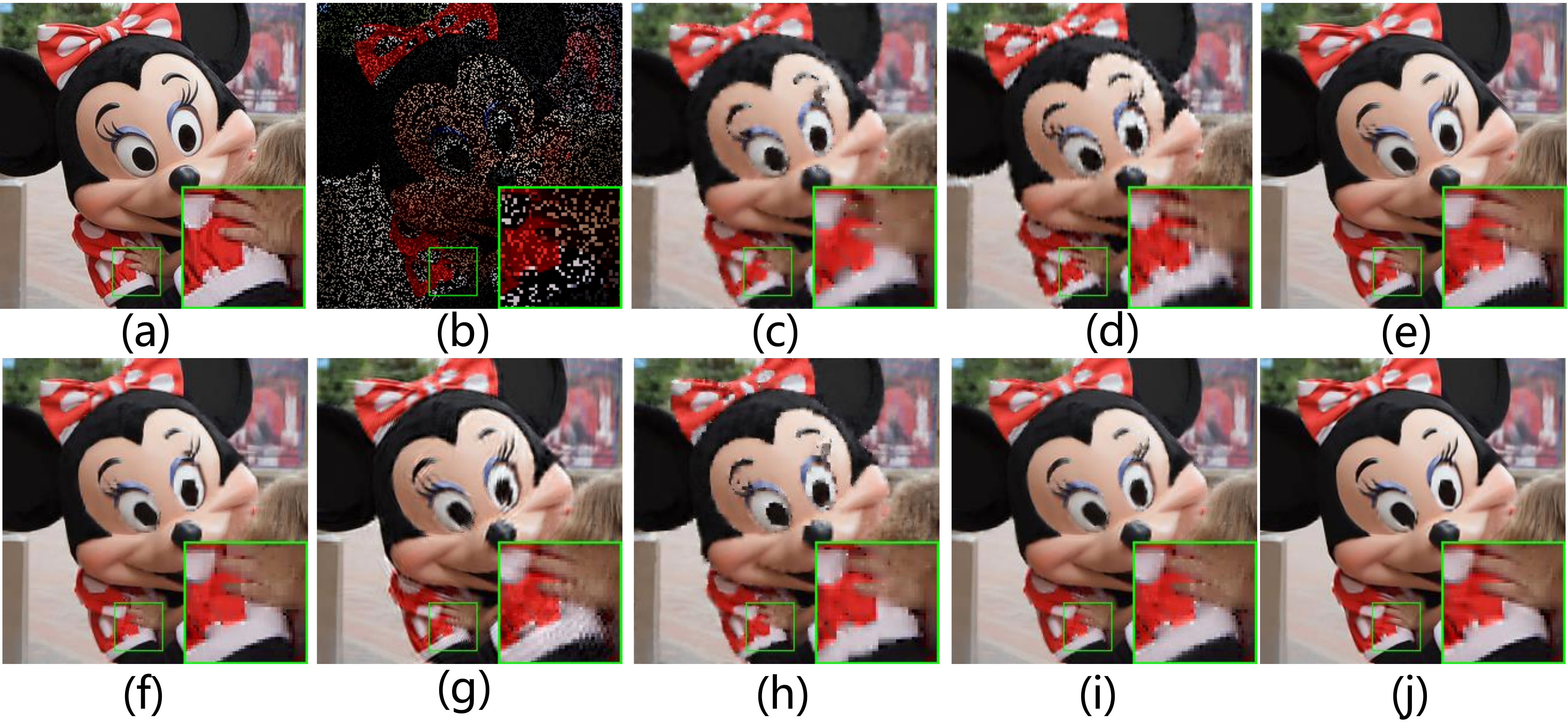}}
	\vspace{-3mm}
	\caption{Inpainting performance comparison on the image $\emph{Mickey}$. (a) Original image; (b) Degraded image with 80\% pixels missing sample; (c) SALSA \cite{54} (PSNR = 28.98dB, SSIM = 0.9243); (d) BPFA \cite{44} (PSNR = 29.43dB, SSIM = 0.9312); (e) IPPO \cite{55} (PSNR = 32.74dB, SSIM = 0.9606); (f) JSM \cite{56} (PSNR = 31.96dB, SSIM = 0.9537); (g) Aloha \cite{57} (PSNR = 30.33dB, SSIM = 0.9371);  (h) NGS \cite{58} (PSNR = 29.75dB, SSIM = 0.9386);   (i) GSR-NNM (PSNR = 29.52dB, SSIM = 0.9303); (j) GSR-WNNM (PSNR = \textbf{33.67dB}, SSIM = \textbf{0.9651}).}
	\label{fig:6}
	\vspace{-3mm}
\end{figure}

\begin{figure}[!htbp]
\vspace{-3mm}
	\centerline{\includegraphics[width=9cm]{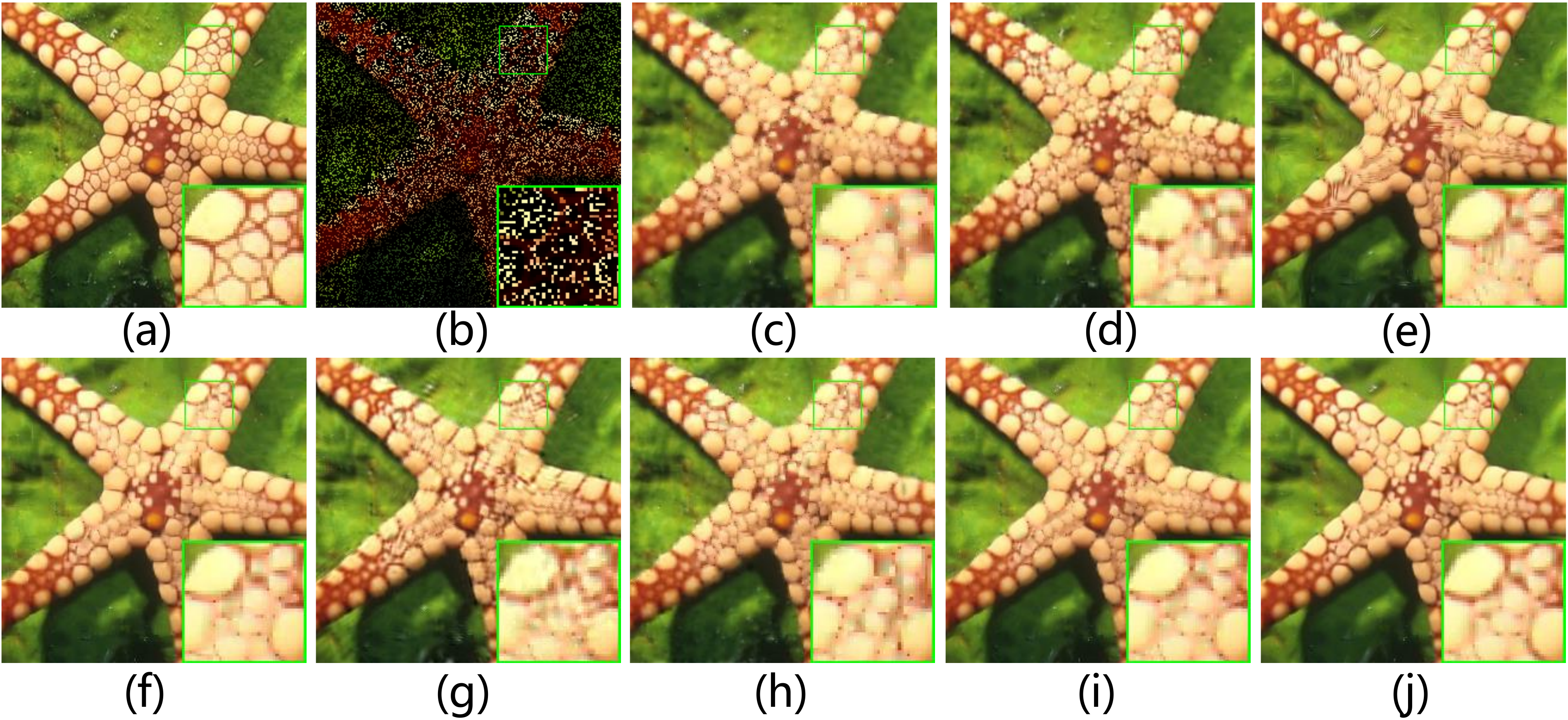}}
	\vspace{-3mm}
	\caption{Inpainting performance comparison on the image $\emph{Starfish}$. (a) Original image; (b) Degraded image with 80\% pixels missing sample; (c) SALSA \cite{54} (PSNR = 30.90dB, SSIM = 0.9335); (d) BPFA \cite{44} (PSNR = 33.13dB, SSIM = 0.9510); (e) IPPO \cite{55} (PSNR = 33.10dB, SSIM = 0.9531); (f) JSM \cite{56} (PSNR = 33.24dB, SSIM = 0.9518); (g) Aloha \cite{57} (PSNR = 31.85dB, SSIM = 0.9418);  (h) NGS \cite{58} (PSNR = 32.10dB, SSIM = 0.9464);   (i) GSR-NNM (PSNR = 30.78dB, SSIM = 0.9232); (j) GSR-WNNM (PSNR = \textbf{34.27dB}, SSIM = \textbf{0.9535}).}
	\label{fig:7}
	\vspace{-3mm}
\end{figure}

\begin{figure}[!htbp]
\vspace{-3mm}
	\centerline{\includegraphics[width=9cm]{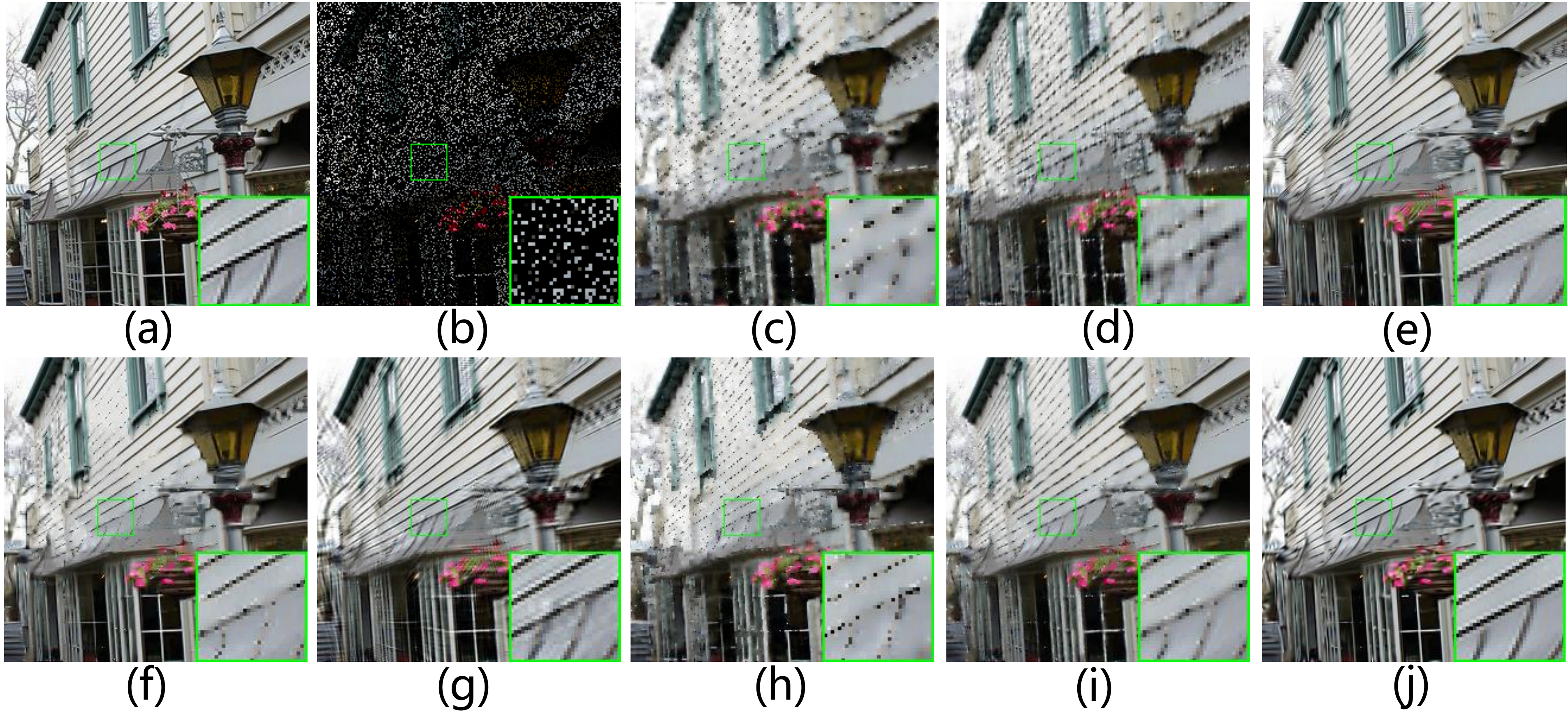}}
	\vspace{-3mm}
	\caption{Inpainting performance comparison on the image $\emph{Light}$. (a) Original image; (b) Degraded image with 80\% pixels missing sample; (c) SALSA \cite{54} (PSNR = 21.47dB, SSIM = 0.8069); (d) BPFA \cite{44} (PSNR = 25.71dB, SSIM = 0.9164); (e) IPPO \cite{55} (PSNR = 26.70dB, SSIM = 0.9350); (f) JSM \cite{56} (PSNR = 26.48dB, SSIM = 0.9322); (g) Aloha \cite{57} (PSNR = 25.84dB, SSIM = 0.9212);  (h) NGS \cite{58} (PSNR = 24.62dB, SSIM = 0.8983);   (i) GSR-NNM (PSNR = 24.62dB, SSIM = 0.8926); (j) GSR-WNNM (PSNR = \textbf{27.28dB}, SSIM = \textbf{0.9447}).}
	\label{fig:8}
	\vspace{-3mm}
\end{figure}

We compare the proposed GSR-WNNM with seven other competing methods, including SALSA \cite{54}, BPFA \cite{44}, IPPO \cite{55}, JSM \cite{56}, Aloha \cite{57}, NGS \cite{58} and GSR-NNM methods. We evaluate these competing methods on a collection of 14 color test images, whose scenes are illustrated in  Fig.~\ref{fig:2}. The PSNR and SSIM results of these competing methods are shown in Table~\ref{Tab:3} and Table~\ref{Tab:4}, respectively. It can be seen that the proposed GSR-WNNM can consistently outperforms other competing methods. In terms of PSNR, the proposed GSR-WNNM achieves 3.32dB, 1.72dB, 0.88dB, 1.02dB, 1.35dB, 2.37dB and 2.20dB improvements on average over SALSA, BPFA, IPPO, JSM, Aloha, NGS and GSR-NNM, respectively. The visual quality comparisons of image $\emph{Mickey}$, $\emph{Starfish}$ and $\emph{Light}$ with 80\% pixels missing are shown in in Fig.~\ref{fig:6}, Fig.~\ref{fig:7} and Fig.~\ref{fig:8}, respectively.  It can be seen that SALSA, NGS and GSR-NNM cannot reconstruct sharp edges and fine details. The BPFA, IPPO, JSM and Aloha methods produce a much better visual quality than SALSA, NGS and GSR-NNM, but they still suffer from some undesirable artifacts, such as the ringing effects. The proposed GSR-WNNM not only preserves sharper edges and finer details, but also eliminates the ringing effects. The better performance of GSR-WNNM is attributed to the singular values have clear physical meanings, for the weight of each group, large singular values of each group usually present major edge and texture information, and vice versa. Therefore, we usually shrink large singular values less, while shrinking smaller ones more, which offers a powerful prior to characterize the sparsity property of natural image signals.

\section {Conclusion}
\label{sec:6}
This paper proposed a scheme to analyze WNNM and NNM from the perspective of the group spare representation (GSR). We designed an adaptive dictionary learning method to bridge the gap between the GSR and the rank minimization models. Based on this adaptive dictionary, we proved that NNM and WNNM are equivalent to the  $\ell_1$-norm minimization based on GSR and the weighted $\ell_1$-norm minimization based on GSR, respectively. Following this, we introduced a mathematical derivation to explain why WNNM is more feasible than NNM. Moreover, due to the heuristical set of the weight in WNNM model,  it sometimes popped out error in the operation of SVD, and thus we presented an adaptive weight setting scheme to avoid this error. We employed the proposed scheme on two low-level vision tasks, \ie, image denoising and image inpainting. Experimental results have demonstrated that WNNM is more feasible than NNM and the proposed scheme can outperform many current state-of-the-art methods both quantitatively and qualitatively.

\end{document}